\documentclass{article}
\usepackage[utf8]{inputenc}
\pdfoutput=1
\usepackage[sc,osf]{mathpazo}
\usepackage[margin=1.1in,letterpaper]{geometry}
\usepackage[sort&compress,numbers]{natbib}
\usepackage[colorlinks=true,citecolor=blue,breaklinks]{hyperref}
\usepackage[hyphenbreaks]{breakurl} 

\makeatletter
\newlength\aftertitskip     \newlength\beforetitskip
\newlength\interauthorskip  \newlength\aftermaketitskip

\def\maketitle{\par
 \begingroup
   \def\thefootnote{\fnsymbol{footnote}}
   \def\@makefnmark{\hbox to 4pt{$^{\@thefnmark}$\hss}}
   \@maketitle \@thanks
 \endgroup
\setcounter{footnote}{0}
 \let\maketitle\relax \let\@maketitle\relax
 \gdef\@thanks{}\gdef\@author{}\gdef\@title{}\let\thanks\relax}

\def\@startauthor{\noindent \normalsize\bf}
\def\@endauthor{}
\def\@starteditor{\noindent \small {\bf Editor:~}}
\def\@endeditor{\normalsize}
\def\@maketitle{\vbox{\hsize\textwidth
 \linewidth\hsize \vskip \beforetitskip
 {\begin{center} \LARGE\@title \par \end{center}} \vskip \aftertitskip
 {\def\and{\unskip\enspace{\rm and}\enspace}%
  \def\addr{\small\it}%
  \def\email{\hfill\small\tt}%
  \def\name{\normalsize\bf}%
  \def\AND{\@endauthor\rm\hss \vskip \interauthorskip \@startauthor}
  \@startauthor \@author \@endauthor}
}}

\makeatother

\pdfoutput=1                    

\usepackage{floatrow} 
\newfloatcommand{capbtabbox}{table}[][\FBwidth]

\usepackage{graphicx}
\usepackage{amsmath,amsthm,amssymb,bm} 
\usepackage{amsfonts}
\usepackage{xspace}
\usepackage{array}
\usepackage{enumerate}
\usepackage{algorithm}
\usepackage{algorithmic}
\usepackage{stmaryrd}
\usepackage{booktabs}
\usepackage{caption}
\usepackage{subcaption}
\usepackage{multirow}

\numberwithin{equation}{section}

\newcommand{\inv}[1]{{#1}^{-1}}

          \newcommand{\ck}{\mathcal{K}}    \newcommand{\co}{\mathcal{O}}         \newcommand{\cy}{\mathcal{Y}} 

\newcommand{\sgn}{\textrm{sgn}}

\theoremstyle{plain}
\newtheorem{theorem}{Theorem}
\newtheorem{corollary}[theorem]{Corollary}

\newtheorem{lemma}[theorem]{Lemma}

\theoremstyle{definition}

\theoremstyle{remark}
\newtheorem{remark}[theorem]{Remark}

\newcommand{\refalgo}[1]{Alg.~\ref{#1}}

\newcommand{\reftab}[1]{Tab.~\ref{#1}}

\newcommand{\refthm}[1]{Thm.~\ref{#1}}
\newcommand{\reflem}[1]{Lemma~\ref{#1}}
\newcommand{\refcor}[1]{Corr.~\ref{#1}}
\newcommand{\eps}{\varepsilon}

\newcommand{\leftgr}[1]{{#1}^{\text{lr}}} 
\newcommand{\rightgr}[1]{{#1}^{\text{rr}}} 
\newcommand{\lobatto}[1]{{#1}^{\text{lo}}} 

\newcommand{\nlsum}{\sum\nolimits}
\newcommand{\set}[1]{\left\lbrace #1\right\rbrace}

\newcommand{\dpp}{\textsc{Dpp}\xspace}

\newcommand{\dg}{\textsc{Dg}\xspace}

\title{Gauss quadrature for matrix inverse forms \\with applications}
\author{\name Chengtao Li \email{ctli@mit.edu}\\
  \name Suvrit Sra \email{suvrit@mit.edu}\\
  \name Stefanie Jegelka \email{stefje@csail.mit.edu}\\
  \addr{Massachusetts Institute of Technology, Cambridge, MA 02139} 
}

\begin{document}
\maketitle

\begin{abstract}
  We present a framework for accelerating a spectrum of machine learning algorithms that require computation of \emph{bilinear inverse forms} $u^\top A^{-1}u$, where $A$ is a positive definite matrix and $u$ a given vector. Our framework is built on Gauss-type quadrature and easily scales to large, sparse matrices. Further, it allows retrospective computation of lower and upper bounds on $u^\top A^{-1}u$, which in turn accelerates several algorithms. We prove that these bounds tighten iteratively and converge at a linear (geometric) rate. To our knowledge, ours is the first work to demonstrate these key properties of Gauss-type quadrature, which is a classical and deeply studied topic. We illustrate empirical consequences of our results by using quadrature to accelerate machine learning tasks involving determinantal point processes and submodular optimization, and observe tremendous speedups in several instances.
\end{abstract}

\section{Introduction}
Symmetric positive definite matrices arise in many areas in a variety of guises: covariances, kernels, graph Laplacians, or otherwise. A basic computation with such matrices is evaluation of the bilinear form $u^Tf(A)v$, where $f$ is a matrix function and $u$, $v$ are given vectors. If $f(A)=A^{-1}$, we speak of computing a \emph{bilinear inverse form (BIF)} $u^T\inv{A}v$. For example, with $u{=}v{=}e_i$ ($i^{\text{th}}$ canonical vector)  $u^Tf(A)v=(A^{-1})_{ii}$ is the $i^{\text{th}}$ diagonal entry of the inverse.

In this paper, we are interested in efficiently computing BIFs, primarily due to their importance in several machine learning contexts, e.g., evaluation of Gaussian density at a point, the Woodbury matrix inversion lemma, implementation of MCMC samplers for Determinantal Point Processes (\dpp), computation of graph centrality measures, and greedy submodular maximization (see Section~\ref{sec:motiv.app}).

When $A$ is large, it is preferable to compute $u^T\inv{A}v$ iteratively rather than to first compute $\inv{A}$ (using Cholesky) at a cost of $\co(N^3)$ operations. One could think of using conjugate gradients to solve $Ax=v$ approximately, and then obtain $u^T\inv{A}v=u^Tx$. But several applications require precise bounds on numerical estimates to $u^\top A^{-1}v$ (e.g., in MCMC based \dpp samplers such bounds help decide whether to accept or reject a transition in each iteration--see Section~\ref{sec:mcdpp}), which necessitates a more finessed approach.

Gauss quadrature is one such approach. Originally proposed in \citep{gauss1815methodus} for approximating integrals, Gauss- and \emph{Gauss-type quadrature} (i.e., Gauss-Lobatto \citep{lobatto1852lessen} and Gauss-Radau \citep{radau1880etude} quadrature) have since found application to bilinear forms including computation of  $u^T\inv{A}v$ \citep{bai1996some}. \citeauthor{bai1996some} also show that Gauss and (right) Gauss-Radau quadrature yield lower bounds, while Gauss-Lobatto and (left) Gauss-Radau yield upper bounds on the BIF $u^T\inv{A}v$.

However, despite its long history and voluminous existing work (see e.g.,~\citep{golub2009matrices}), our understanding of Gauss-type quadrature for matrix problems is far from complete. For instance, it is not known whether the bounds on BIFs improve with more quadrature iterations; nor is it known how the bounds obtained from Gauss, Gauss-Radau and Gauss-Lobatto quadrature compare with each other. \emph{We do not even know how fast the iterates of Gauss-Radau or Gauss-Lobatto quadrature converge.} 

\textbf{Contributions.} We address all the aforementioned problems and make the following  main contributions:
\begin{list}{--}{\leftmargin=1em}\vspace*{-7pt}
\setlength{\itemsep}{0pt}
\item We show that the lower and upper bounds generated by Gauss-type quadrature monotonically approach the target value~(Theorems~\ref{thm:lowbtwn} and \ref{thm:upbtwn}; \refcor{cor:monlow}). Furthermore, we show that for the same number of iterations, Gauss-Radau quadrature yields bounds superior to those given by Gauss or Gauss-Lobatto, but somewhat surprisingly all three share the same convergence rate. 
\item We prove linear convergence rates for Gauss-Radau and Gauss-Lobatto explicitly~(Theorems~\ref{thm:rrconv} and \ref{thm:lrconv}; \refcor{cor:loconv}).
\item We demonstrate implications of our results for two tasks: (i) scalable Markov chain sampling from a \dpp; and (ii) running a greedy algorithm for submodular optimization. In these applications, quadrature accelerates computations, and the bounds aid early stopping. 
\end{list}
\vspace{-8pt}
Indeed, on large-scale sparse problems our methods lead to even several orders of magnitude in speedup.

\paragraph{Related Work.} There exist a number of methods for efficiently approximating matrix bilinear forms. \citet{brezinski1999error} and \citet{brezinski2012estimations} use extrapolation of matrix moments and interpolation to estimate the 2-norm error of linear systems and the trace of the matrix inverse. \citet{fika2014estimates} extend the extrapolation method to BIFs and show that the derived one-term and two-term approximations coincide with Gauss quadrature, hence providing lower bounds. Further generalizations address $x^* f(A) y$ for a Hermitian matrix $A$ \citep{fika2015estimation}. In addition, other methods exist for estimating trace of a matrix function~\citep{bai1996bounds,brezinski2012estimations,fika2015stochastic} or diagonal elements of matrix inverse~\citep{bekas2007estimator,tang2012probing}.

Many of these methods may be applied to computing BIFs. But they do not provide intervals bounding the target value, just approximations. Thus, a black-box use of these methods may change the execution of an algorithm whose decisions (e.g., whether to transit in a Markov Chain) rely on the BIF value to be within a specific interval. Such changes can break the correctness of the algorithm.

Our framework, in contrast, yields iteratively tighter lower and upper bounds (Section~\ref{sec:main}), so the algorithm is guaranteed to make correct decisions (Section~\ref{sec:algos}).

\section{Motivating Applications}
\label{sec:motiv.app}
BIFs are important to numerous problems. We recount below several notable examples: in all cases, efficient computation of bounds on BIFs is key to making the algorithms practical.

\textbf{Determinantal Point Processes.}
A determinantal point process (\dpp) is a distribution over subsets of a set $\cy$ ($|\cy|=N$). In its \emph{L-ensemble} form, a \dpp uses a positive semidefinite kernel $L\in\mathbb{R}^{N\times N}$, and to a set $Y \subseteq \cy$ assigns probability $P(Y) \propto \det(L_Y)$ where $L_Y$ is the submatrix of $L$ indexed by entries in $Y$. If we restrict to $|Y| = k$, we obtain a $k$-\dpp. \dpp's are widely used in machine learning, see e.g., the survey~\citep{kulesza2012determinantal}. 

Exact sampling from a ($k$-)\dpp requires eigendecomposition of $L$ \citep{hough2005}, which is prohibitive. For large $N$, Metropolis Hastings (MH) or Gibbs sampling are preferred and state-of-the-art. Therein the core task is to compute transition probabilities -- an expression involving BIFs -- which are compared with a random scalar threshold. 

For MH \cite{belabbas2009spectral, anari2016monte}, the transition probabilities from a current subset (state) $Y$ to $Y'$ are $\min\{1,L_{u,u}-L_{u,Y}L_Y^{-1}L_{Y,u}\}$ for $Y'=Y\cup\{u\}$; and $\min\{1,L_{u,u}-L_{u,Y'}L_{Y'}^{-1}L_{Y',u}\}$ for $Y'=Y\backslash\{u\}$. In a $k$-\dpp, the moves are swaps with transition probabilities $\min\Big\{1, {L_{u,u} - L_{u,Y'}L_{Y'}^{-1}L_{Y',u} \over L_{v,v} - L_{v,Y'}L_{Y'}^{-1}L_{Y',v}}\Big\}$ for replacing $v \in Y$ by $u \notin Y$ (and $Y'=Y\backslash \{v\}$). We illustrate this application in greater detail in Section~\ref{sec:mcdpp}. 

\dpp{}s are also useful for (repulsive) priors in Bayesian models~\citep{rockovadeterminantal,kwok2012priors}. Inference for such latent  variable models uses Gibbs sampling, which again involves BIFs. 

\textbf{Submodular optimization, Sensing.}
Algorithms for maximizing submodular functions can equally benefit from efficient BIF bounds. Given a positive definite matrix $K \in \mathbb{R}^{N \times N}$, the set function $F(S) = \log \det(K_S)$ is \emph{submodular}: 
 for all $S \subseteq T \subseteq [N]$ and $i \in [N]\setminus T$, it holds that $F(S \cup \{i\}) - F(S) \geq F(T \cup \{i\}) - F(T)$. 

Finding the set $S^* \subseteq [N]$ that maximizes $F(S)$ is a key task for MAP inference with \dpp{}s \citep{gillenwater12}, matrix approximations by column selection \citep{boutsidis2009improved,sviridenko15} and sensing \citet{krause2008}. For the latter, we model spatial phenomena (temperature, pollution) via Gaussian Processes and select locations to maximize the joint entropy $F_1(S) = H(X_S) = \log \det(K_S) + \mathrm{const}$ of the observed variables, or the mutual information $F_2(S) = I(X_S; X_{[N]\setminus S})$ between observed and unobserved variables. 

Greedy algorithms for maximizing monotone \citep{nemhauser1978} or non-monotone \citep{buchbinder12} submodular functions rely on marginal gains of the form
\begin{align*}
  F_1(S \cup \{i\}) - F_1(S) &= \log( K_{i} - K_{iS}K^{-1}_S K_{Si});\\
  F_1(T \setminus \{i\}) - F_1(T) &= -\log( K_{i} - K_{iU}K^{-1}_U K_{Ui});\\
  F_2(S \cup \{i\}) - F_2(S) &= \log\tfrac{K_i - K_{iS}K_{S}^{-1}K_{Si}}{K_i - K_{i\bar{S}}K_{\bar{S}}^{-1}K_{\bar{S}i}}
\end{align*}
for $U = T \backslash\{i\}$ and $\bar{S} = [N]\backslash S$. The algorithms compare those gains to a random threshold, or find an item with the largest gain. In both cases, efficient BIF bounds offer speedups. They can be combined with lazy~\citep{minoux78} and stochastic greedy algorithms~ \citep{mirzasoleiman15}.

\textbf{Network Analysis, Centrality.} When analyzing relationships and information flows between connected entities in a network, such as people, organizations, computers, smart hardwares, etc.~\citep{scott2012social,jure08,atzori2010internet,fenu2013network,estrada2010network,benzi2013total}, 
an important question is to measure popularity, centrality, or importance of a node.

 Several existing popularity measures can be expressed as the solution to a large-scale linear system. For example, \emph{PageRank}~\citep{page99rank} is the solution to $(I - (1 - \alpha)A^\top)x = \alpha \bold{1}/N$, and \emph{Bonacich centrality}~\citep{bonacich87} is the solution to $(I - \alpha A)x = \bold{1}$,
where $A$ is the adjacency matrix. 
When computing local estimates, i.e., only a few entries of $x$, we obtain exactly the task of computing BIFs~\citep{wasow1952note,christina2014solving}. Moreover, we may only need local estimates to an accuracy sufficient for determining which entry is larger, a setting where our quadrature based bounds on BIFs will be useful.

\textbf{Scientific Computing.} In computational physics BIFs are used for estimating selected entries of the inverse of a large sparse matrix. More generally, BIFs can help in estimating the trace of the inverse, a computational substep in lattice Quantum Chromodynamics~\citep{dong1994stochastic,frommer2000numerical}, some signal processing tasks~\citep{golub2008approximation}, and in Gaussian Process (GP) Regression~\cite{rasmussen}, e.g., for estimating variances. 
In numerical linear algebra, BIFs are used in rational approximations~\citep{sidje2011rational}, evaluation of Green's function~\citep{freericks2006transport}, and selective inversion of sparse matrices~\citep{lin2011fast,lin2011selinv,christina2014solving}. A notable use is the design of preconditioners~\citep{benzi1999bounds} and  uncertainty quantification~\citep{bekas2009low}. 

\textbf{Benefiting from fast iterative bounds.} Many of the above examples use BIFs to rank values, to identify the largest value or compare them to a scalar or to each other. In such cases, we first compute fast, crude lower and upper bounds on a BIF, refining iteratively, just as far as needed to determine the comparison. Figure~\ref{fig:conv} in Section~\ref{sec:conv} illustrates the evolution of these bounds, and Section~\ref{sec:algos} explains details.

\section{Background on Gauss Quadrature}
\label{sec:gauss}
For convenience, we begin by recalling key aspects of Gauss quadrature,\footnote{The summary in this section is derived from various sources:~\cite{gautschi1981survey,bai1996some,golub2009matrices}. Experts can skim this section for collecting our notation before moving onto Section~\ref{sec:main}, which contains our new results.} as applied to  computing $u^\top f(A) v$, for an $N\times N$ symmetric positive definite matrix $A$ that has \emph{simple} eigenvalues, arbitrary vectors $u, v$, and a matrix function $f$. For a more detailed account of the relevant background on Gauss-type quadratures please refer to Appendix~\ref{append:sec:gauss}, or \citep{golub2009matrices}.

It suffices to consider  $u^\top f(A)u$ thanks to the identity
\begin{equation*}
  u^\top f(A)v = \tfrac14(u+v)^\top f(A) (u+v) - \tfrac14(u-v)^\top f(A)(u-v).
\end{equation*}
Let $A = Q^\top \Lambda Q$ be the eigendecomposition of $A$ where $Q$ is orthonormal. Letting $\tilde{u} = Qu$, we then have
\begin{equation*}
  u^\top f(A) u = \tilde{u}^\top f(\Lambda)\tilde{u} = \nlsum_{i=1}^N f(\lambda_i) \tilde{u}_i^2.
\end{equation*}
Toward computing $u^Tf(A)u$, a key conceptual step is to write the above sum as the Riemann-Stieltjes integral
\begin{equation}
  \label{eq:1}
  I[f] := u^\top f(A)u = \int_{\lambda_{\min}}^{\lambda_{\max}} f(\lambda)d\alpha(\lambda),
\end{equation}
where $\lambda_{\min} \in (0, \lambda_1)$, $\lambda_{\max} > \lambda_{N}$, and $\alpha(\lambda)$ is piecewise constant measure defined by
\begin{align*}
\alpha(\lambda) := \left\{
\begin{array}{lll}
0, & \lambda < \lambda_1, & \\
\sum_{j=1}^k \tilde{u}_j^2, & \lambda_k\le \lambda < \lambda_{k+1}, & k < N,\\
\sum_{j=1}^N \tilde{u}_j^2, & \lambda_N\le \lambda.
\end{array}
\right.
\end{align*}
Our task now reduces to approximating the integral~\eqref{eq:1}, for which we invoke the powerful idea of Gauss-type quadratures~\citep{gauss1815methodus, radau1880etude,lobatto1852lessen,gautschi1981survey}. We rewrite the integral~\eqref{eq:1} as 
\begin{equation}
  \label{eq:2}
  I[f] := Q_{n} + R_{n} = \nlsum_{i=1}^n \omega_i f(\theta_i) + \nlsum_{i=1}^m \nu_i f(\tau_i) + R_{n}[f],
\end{equation}
where $Q_n$ denotes the $n$th degree approximation and $R_n$ denotes the remainder term. 
In representation~\eqref{eq:2} the \emph{weights} $\set{\omega_i}_{i=1}^n$, $\set{\nu_i}_{i=1}^m$, and quadrature \emph{nodes} $\set{\theta_i}_{i=1}^n$ are unknown, while the values $\set{\tau_i}_{i=1}^m$ are prescribed and lie outside the interval of integration $(\lambda_{\min}, \lambda_{\max})$. 

Different choices of these parameters yield different quadrature rules: $m = 0$ gives  Gauss quadrature~\cite{gauss1815methodus}; $m = 1$ with $\tau_1 = \lambda_{\min}$~($\tau_1 = \lambda_{\max}$) gives left~(right) Gauss-Radau quadrature~\cite{radau1880etude}; $m = 2$ with $\tau_1 = \lambda_{\min}$ and $\tau_2 = \lambda_{\max}$ yields Gauss-Lobatto quadrature~\cite{lobatto1852lessen}; while for general $m$ we obtain Gauss-Christoffel quadrature~\citep{gautschi1981survey}.

The weights $\set{\omega_i}_{i=1}^n$, $\set{\nu_i}_{i=1}^m$ and nodes $\set{\theta_i}_{i=1}^n$ are chosen such that if $f$ is a polynomial of degree less than $2n+m-1$, then the interpolation $I[f] = Q_{n}$ is \emph{exact}. For Gauss quadrature, we can recursively build the \emph{Jacobi matrix}
\begin{align}
  J_n &= {\footnotesize\begin{bmatrix}
    \alpha_1 & \beta_1 & & &\\
    \beta_1 & \alpha_2 & \beta_2 & & \\
    & \beta_2 & \ddots & \ddots &\\ 
    & & \ddots & \alpha_{n-1} & \beta_{n-1}\\
    & & & \beta_{n-1} & \alpha_n
  \end{bmatrix}},
\end{align}
and obtain from its spectrum the desired weights and nodes. Theorem~\ref{thm:gauss} makes this more precise.
\begin{theorem}\cite{wilf1962mathematics,golub1969calculation}\label{thm:gauss}
  The eigenvalues of $J_n$ form the nodes $\{\theta_i\}_{i=1}^n$ of Gauss quadrature; the weights $\{\omega_i\}_{i=1}^n$ are given by the squares of the first components of the eigenvectors of $J_n$.
\end{theorem}

If $J_n$ has the eigendecomposition $P_n^\top \Gamma P_n$, then for Gauss~quadrature \refthm{thm:gauss} yields 
\begin{equation}
  \label{eq:4}
Q_{n} = \nlsum_{i=1}^n \omega_i f(\theta_i) = e_1^\top P_n^\top f(\Gamma) P_n e_1 = e_1^\top f(J_n) e_1.
\end{equation}
Given $A$ and $u$, our task is to compute $Q_n$ and the Jacobi matrix $J_n$. For BIFs, we have that $f(J_n)=J_n^{-1}$, so \eqref{eq:4} becomes $Q_n=e_1^TJ_n^{-1}e_1$, which can be computed recursively using the Lanczos algorithm~\citep{lanczos1950iteration}. 
For Gauss-Radau and Gauss-Lobatto quadrature we can compute modified versions of Jacobi matrices $\leftgr{J}_n$~(for left Gauss-Radau), $\rightgr{J}_n$~(for right Gauss-Radau) and $\lobatto{J}_n$~(for Gauss-Lobatto) based on $J_n$. 
The corresponding nodes and weights, and thus the approximation of Gauss-Radau and Gauss-Lobatto quadratures, are then obtained from these modified Jacobi matrices, similar to Gauss quadrature. 
Aggregating all these computations yields an algorithm that iteratively obtains bounds on $u^T\inv{A}u$. 
The combined procedure, \emph{Gauss Quadrature Lanczos (GQL)}~\citep{golub1997matrices}, is summarily presented as Algorithm~\ref{algo:gql}. 
The complete algorithm may be found in Appendix~\ref{append:sec:gauss}.

\begin{theorem}\cite{meurant1997computation}\label{lem:bounds}
  Let $g_i$, $\leftgr{g}_i$, $\rightgr{g}_i$, and $\lobatto{g}_i$ be the $i$-th iterates of Gauss, left Gauss-Radau, right Gauss-Radau, and Gauss-Lobatto quadrature, respectively, as computed by Alg.~\ref{algo:gql}. Then, $g_i$ and $\rightgr{g}_i$ provide lower bounds on $u^\top A^{-1} u$, while $\leftgr{g}_i$ and $\lobatto{g}_i$ provide upper bounds.
\end{theorem}

\begin{algorithm}[t]\small
  \caption{Gauss Quadrature Lanczos (GQL)}\label{algo:gql}
  \begin{algorithmic}
    \STATE{\textbf{Input:} Matrix $A$, vector $u$; lower and upper bounds $\lambda_{\min}$ and $\lambda_{\max}$ on the spectrum of $A$}
    \STATE{\textbf{Output:} $(g_i,\rightgr{g}_i,\leftgr{g}_i,\lobatto{g}_i)$: Gauss, right Gauss-Radau, left Gauss-Radau, and Gauss-Lobatto quadrature estimates for each $i$}
    \STATE \textbf{Initialize}: $u_0 = u / \|u\|$, $g_1 = \|u\|/u_0^\top A u_0$, $i = 2$
    \FOR{$i=1$ to $N$}
      \STATE Update $J_i$ using a Lanczos iteration
      \STATE Solve for the modified Jacobi matrices $\leftgr{J}_i$, $\rightgr{J}_i$ and $\lobatto{J}_i$.
      \STATE Compute $g_i$, $\rightgr{g}_i$, $\leftgr{g}_i$ and $\lobatto{g}_i$ with Sherman-Morrison formula.
      \ENDFOR 
  \end{algorithmic}
\end{algorithm}

It turns out that the bounds given by Gauss quadrature have a close relation to the approximation error of conjugate gradient (CG) applied to a suitable problem. Since we know the convergence rate of CG, we can obtain from it the following estimate on the \emph{relative error} of Gauss quadrature.
\begin{theorem}[Relative error Gauss quadrature]\label{thm:gaussconv}
The $i$-th iterate of Gauss quadrature satisfies the relative error bound
\begin{align}
  \frac{g_N - g_i}{g_N} \le 2\Bigl({\sqrt{\kappa} - 1\over \sqrt{\kappa} + 1}\Bigr)^i,
\end{align}
where $\kappa:=\lambda_1(A)/\lambda_N(A)$ is the condition number of $A$.
\end{theorem}
\noindent In other words, \refthm{thm:gaussconv} shows that the iterates of Gauss quadrature have a linear (geometric) convergence rate.

\section{Main Theoretical Results}
\label{sec:main}
In this section we summarize our main theoretical results. As before, detailed proofs may be found in~Appendix~\ref{app:sec:proofs}. The key questions that we answer are: (i) do the bounds on $u^\top A^{-1} u$ generated by GQL improve monotonically with each iteration; (ii) how tight are these bounds; and (iii) how fast do Gauss-Radau and Gauss-Lobatto iterations converge? Our answers not only fill gaps in the literature on quadrature, but provide a theoretical base for speeding up algorithms for some applications (see Sections~\ref{sec:motiv.app} and \ref{sec:algos}).

\subsection{Lower Bounds}
Our first result shows that both Gauss and right Gauss-Radau quadratures give iteratively better lower bounds on $u^\top A^{-1} u$. Moreover, with the same number of iterations, right Gauss-Radau yields tighter bounds.
\begin{theorem}\label{thm:lowbtwn}
  Let $i < N$. Then, $\rightgr{g}_i$ yields better bounds than $g_i$ but worse bounds than $g_{i+1}$; more precisely,
\begin{equation*}
g_i\le \rightgr{g}_i\le g_{i+1},\;\;i < N.
\end{equation*}
\end{theorem}

Combining Theorem~\ref{thm:lowbtwn} with the convergence rate of relative error for Gauss quadrature (Thm.~\ref{thm:gaussconv}) we obtain the following convergence rate estimate for right Gauss-Radau. 
\begin{theorem}[Relative error right Gauss-Radau]
\label{thm:rrconv}
  For each iteration $i$, the right Gauss-Radau iterate $\rightgr{g}_i$ satisfies
\begin{equation*}
  \frac{g_N - \rightgr{g}_i}{g_N} \le 2\Bigl(\frac{\sqrt{\kappa} - 1}{\sqrt{\kappa} + 1}\Bigr)^i.
\end{equation*}

\end{theorem}

\subsection{Upper Bounds}
Our second result compares Gauss-Lobatto with left Gauss-Radau quadrature.
\begin{theorem}\label{thm:upbtwn}
  Let $i < N$. Then, $\leftgr{g}_i$ gives better upper bounds than $\lobatto{g}_{i}$ but worse than $\lobatto{g}_{i+1}$; more precisely,
  \begin{equation*}
    \lobatto{g}_{i+1} \le \leftgr{g}_i \le \lobatto{g}_i,\quad i < N.
  \end{equation*}
\end{theorem}

This shows that bounds given by both Gauss-Lobatto and left Gauss-Radau become tighter with each iteration. For the same number of iterations, left Gauss-Radau provides a tighter bound than Gauss-Lobatto.

Combining the above two theorems, we obtain the following corollary for all four Gauss-type quadratures.

\begin{corollary}[Monotonicity]\label{cor:monlow}
  With increasing $i$, $g_i$ and $\rightgr{g}_i$ give increasingly better lower bounds and $\leftgr{g}_i$ and $\lobatto{g}_i$ give increasingly better upper bounds, that is,
\begin{align*}
  g_i \le g_{i+1}; \quad&\rightgr{g}_i \le \rightgr{g}_{i+1};\\
  \leftgr{g}_i \ge \leftgr{g}_{i+1}; \quad& \lobatto{g}_i \ge \lobatto{g}_{i+1}.
\end{align*}
\end{corollary}

\subsection{Convergence rates}
Our next two results state linear convergence rates for left Gauss-Radau quadrature and Gauss-Lobatto quadrature applied to computing the BIF $u^T\inv{A}u$.

\begin{theorem}[Relative error left Gauss-Radau]
\label{thm:lrconv}
For each $i$, the left Gauss-Radau iterate $\leftgr{g}_i$ satisfies
\begin{equation*}
  \frac{\leftgr{g}_i - g_N}{g_N} \le 2\kappa^+\Bigl(\frac{\sqrt{\kappa} - 1}{\sqrt{\kappa} + 1}\Bigr)^i,
\end{equation*}
where $\kappa^+ := \lambda_N/ \lambda_{\min},\ i<N$.
\end{theorem}
Theorem~\ref{thm:lrconv} shows that the error again decreases linearly, and it also depends on teh accuracy of $\lambda_{\min}$, our estimate of the smallest eigenvalue that determines the range of integration.
Using the relations between left Gauss-Radau and Gauss-Lobatto, we readily obtain the following corollary.
\begin{corollary}[Relative error Gauss-Lobatto]
\label{cor:loconv}
For each $i$, the Gauss-Lobatto iterate $\lobatto{g}_i$ satisfies
\begin{equation*}
  \frac{\lobatto{g}_i - g_N}{g_N} \le 2\kappa^+\Bigl(\frac{\sqrt{\kappa} - 1}{\sqrt{\kappa} + 1}\Bigr)^{i-1},
\end{equation*}
where $\kappa^+ := \lambda_N/ \lambda_{\min}\text{ and }i<N$.
\end{corollary}

\textbf{Remarks} All aforementioned results assumed that $A$ is strictly positive definite with simple eigenvalues. In~Appendix~\ref{append:sec:general}, we show similar results for the more general case that $A$ is only required to be symmetric, and $u$ lies in the space spanned by eigenvectors of $A$ corresponding to distinct positive eigenvalues.

\subsection{Empirical Evidence}
\label{sec:conv}
Next, we empirically verify our the theoretical results shown above. We generate a random symmetric matrix $A \in\mathbb{R}^{100\times 100}$ with density $10\%$, where each entry is either zero or standard normal, and shift its diagonal entries to make its smallest eigenvalue $\lambda_1=10^{-2}$, thus making $A$ positive definite.
We set $\lambda_{\min} = \lambda_1^- = (\lambda_1 - 10^{-5})$ and $\lambda_{\max} = \lambda_N^+ = (\lambda_N + 10^{-5})$. We randomly sample $u\in\mathbb{R}^{100}$ from a standard normal distribution. Figure~\ref{fig:conv} illustrates how the lower and upper bounds given by the four quadrature rules evolve with the number of iterations. 

Figure~\ref{fig:conv} (b) and (c) show the sensitivity of the rules (except Gauss quadrature) to estimating the extremal eigenvalues. Specifically, we use $\lambda_{\min} = 0.1\lambda_1^-$ and $\lambda_{\max} = 10\lambda_N^+$.

\begin{figure}[h!]
\centering
  	\hspace*{-8pt}
  	\begin{subfigure}{.32\textwidth}
	\centering
  	\includegraphics[width = \textwidth]{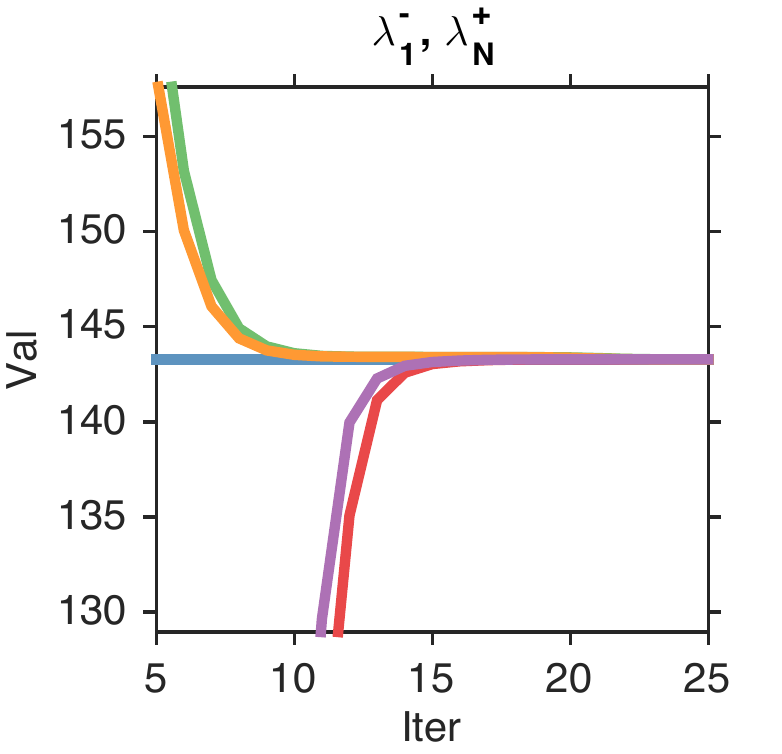}
	\caption{}
	\end{subfigure}%
        \hspace*{-8pt}
	\begin{subfigure}{.32\textwidth}
	\centering
  	\includegraphics[width = \textwidth]{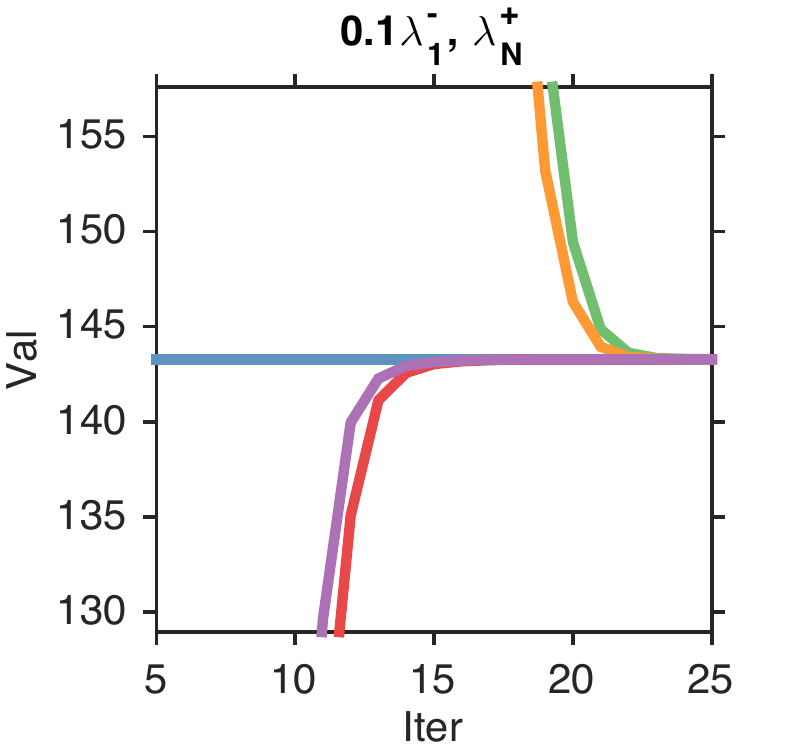}
	\caption{}
	\end{subfigure}%
  	\hspace*{-8pt}
	\begin{subfigure}{.32\textwidth}
	\centering
  	\includegraphics[width = \textwidth]{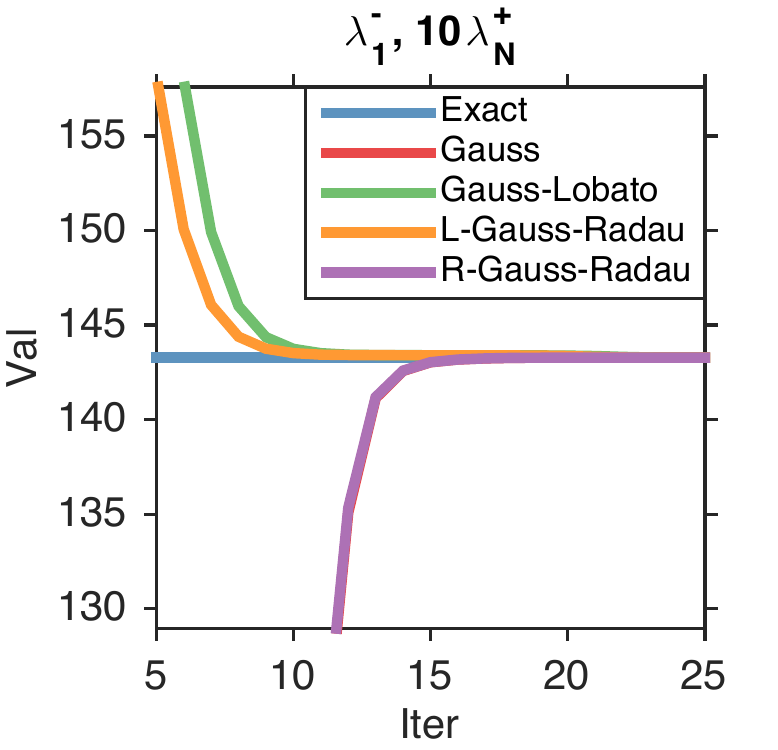}
	\caption{}
	\end{subfigure}
\caption{ Lower and upper bounds computed by Gauss-type quadrature in each iteration on $u^\top\inv{A}u$ with $A\in\mathbb{R}^{100\times 100}$. } 
\label{fig:conv}
\end{figure}

The plots in Figure~\ref{fig:conv} agree with the theoretical results. First, all quadrature rules are seen to yield iteratively tighter bounds. The bounds obtained by the Gauss-Radau quadrature are superior to those given by Gauss and Gauss-Lobatto quadrature (also numerically verified). Notably, the bounds given by all quadrature rules converge very fast -- within 25 iterations they yield reasonably tight bounds. 

It is valuable to see how the bounds are affected if we do not have good approximations to the extremal eigenvalues $\lambda_1$ and $\lambda_N$. Since Gauss quadrature does not depend on the approximations $\lambda_{\min} < \lambda_1$ and $\lambda_{\max} > \lambda_N$, its bounds remain the same in (a),(b),(c). Left Gauss-Radau depends on the quality of $\lambda_{\min}$, and, with a poor approximation takes more iterations to converge (Figure~\ref{fig:conv}(b)). Right Gauss-Radau depends on the quality of $\lambda_{\max}$; thus, if we use $\lambda_{\max}=10\lambda_N^+$ as our approximation, its bounds become worse (Figure~\ref{fig:conv}(c)). However, its bounds are never worse than those obtained by Gauss quadrature.
 
Finally, Gauss-Lobatto depends on both $\lambda_{\min}$ and $\lambda_{\max}$, so its bounds become worse whenever we lack good approximations to $\lambda_1$ or $\lambda_N$. Nevertheless, its quality is lower-bounded by left Gauss-Radau as stated in~\refthm{thm:upbtwn}.

\section{Algorithmic Results and Applications}
\label{sec:algos}
Our theoretical results show that Gauss-Radau quadrature provides good lower and upper bounds to BIFs. More importantly, these bounds get iteratively tighter at a linear rate, finally becoming exact~(see Appendix~\ref{app:sec:proofs}). However, in many applications motivating our work (see Section~\ref{sec:motiv.app}), we do not need exact values of BIFs; bounds that are tight enough suffice for the algorithms to proceed. As a result, all these applications benefit from our theoretical results that provide iteratively tighter bounds. This idea translates into a \emph{retrospective} framework for accelerating methods whose progress relies on knowing an interval containing the BIF. 
Whenever the algorithm takes a step (\emph{transition}) that depends on a BIF (e.g., as in the next section, a state transition in a sampler if the BIF exceeds a certain threshold), we compute rough bounds on its value. If the bounds suffice to take the critical decision (e.g., decide the comparison), then we stop the quadrature. If they do not suffice, we take one or more additional iterations of quadrature to tighten the bound. 
Algorithm~\ref{algo:framework} makes this idea explicit.

\begin{algorithm}\small
  \caption{Efficient Retrospective Framework}\label{algo:framework}
	\begin{algorithmic} 
	\REQUIRE{Algorithm with transitions that depend on BIFs}
	\WHILE{algorithm not yet done}
		\WHILE{no transition request for values of a BIF}
			\STATE proceed with the original algorithm
		\ENDWHILE
		\IF{exist transition request for values of a BIF}
			\WHILE{bounds on the BIF not tight enough to make the transition}
				\STATE Retrospectively run one more iteration of left and(or) right Gauss-Radau to obtain tighter bounds.
			\ENDWHILE
			\STATE Make the correct transition with bounds
		\ENDIF
	\ENDWHILE
\end{algorithmic}
\end{algorithm}

We illustrate our framework by accelerating: (i) Markov chain sampling for ($k$-)\dpp{}s; and (ii) maximization of a (specific) nonmonotone submodular function.

\subsection{Retrospective Markov Chain ($k$-)\dpp}
\label{sec:mcdpp}
First, we use our framework to accelerate iterative samplers for Determinantal Point Processes.
Specifically, we discuss MH sampling \citep{kang2013fast}; the variant for Gibbs sampling follows analogously. 

The key insight is that all state transitions of the Markov chain rely on a comparison between a scalar $p$ and a quantity involving the bilinear inverse form. 
Given the current set $Y$, assume we propose to add element $y$ to $Y$. The probability of transitioning to state $Y \cup \{y\}$ is $q = \min\{1,L_{y,y}-L_{y,Y}L_Y^{-1}L_{Y,y}\}$. To decide whether to accept this transition, we sample $p\sim \mathrm{Uniform(0,1)}$; if $p<q$ then we accept the transition, otherwise we remain at $Y$. Hence, we need to compute $q$ just accurately enough to decide whether $p<q$. To do so, we can use the aforementioned lower and upper bounds on $L_{y,Y}L_Y^{-1}L_{Y,y}$.

Let $s_i$ and $t_i$ be lower and upper bounds for this BIF in the $i$-th iteration of Gauss quadrature.
If $p \le L_{y,y} - t_i$, then we can safely accept the transition, if $p \ge L_{y,y} - s_i$, then we can safely reject the transition.
Only if $L_{y,y}-t_i < p < L_{y,y}-s_i$, we cannot make a decision yet, and therefore 
retrospectively perform one more iteration of Gauss quadrature to obtain tighter upper and lower bounds 
$s_{i+1}$ and $t_{i+1}$. 
We continue until the bounds are sharp enough to safely decide whether to make the transition.
Note that in each iteration we make the same decision as we would with the exact value of the BIF, and hence the resulting algorithm (\refalgo{algo:gaussdpp}) is an exact Markov chain for the \dpp.
In each iteration, it calls~\refalgo{algo:dpp_judge}, which uses step-wise lazy Gauss quadrature for deciding the comparison, while stopping as early as possible.

\begin{algorithm}[t!]\small
	\caption{\small Gauss-\dpp($L$)}\label{algo:gaussdpp}
	\begin{algorithmic} 
	\REQUIRE{\dpp kernel $L$; ground set $\cy$}
	\ENSURE{$Y$ sampled from exact \dpp$(L)$}
	\STATE Randomly Initialize $Y\subseteq \cy$
	\WHILE{chain not mixed}
		\STATE Pick $y\in\cy$, $p\in(0,1)$ uniformly randomly
		\IF{$y\in Y$}
			\STATE $Y' = Y\backslash\{y\}$
			\STATE Compute bounds $\lambda_{\min}$, $\lambda_{\max}$ on the spectrum of $L_{Y'}$
			\IF{\textsc{DppJudge}($L_{yy}{-}{p}$, $L_{Y',y}$, $L_{Y'}$, $\lambda_{\min}$, $\lambda_{\max}$)}
				\STATE $Y = Y'$
			\ENDIF
		\ELSE
			\STATE $Y' = Y\cup \{y\}$
			\STATE Compute bounds $\lambda_{\min}$, $\lambda_{\max}$ on the spectrum of $L_{Y}$
			\IF{\textbf{not} \textsc{DppJudge}($L_{yy}{-}p$, $L_{Y,y}$, $L_{Y}$, $\lambda_{\min}$, $\lambda_{\max}$)}
				\STATE $Y = Y'$
			\ENDIF
		\ENDIF
	\ENDWHILE
\end{algorithmic}
\end{algorithm}

\begin{algorithm}[t!]\small
	\caption{\textsc{DppJudge}($t, u, A, \lambda_{\min}, \lambda_{\max}$)}\label{algo:dpp_judge}
	\begin{algorithmic} 
	\REQUIRE{target value $t$; vector $u$, matrix $A$; lower and upper bounds $\lambda_{\min}$ and $\lambda_{\max}$ on the spectrum of $A$}
	\ENSURE{Return \textbf{true} if $t < u^\top A^{-1} u$, \textbf{false} otherwise}
	\WHILE{{\bf true}}
		\STATE Run one Gauss-Radau iteration to get $\rightgr{g}$ and $\leftgr{g}$ for $u^\top A^{-1}u$.
		\IF{$t < \rightgr{g}$}
			\STATE \textbf{return} \textbf{true}
		\ELSIF{$t \ge \leftgr{g}$}
			\STATE \textbf{return} \textbf{false}
		\ENDIF
		\STATE $i = i + 1$
	\ENDWHILE
\end{algorithmic}
\end{algorithm}

If we condition the \dpp on observing a set of a fixed cardinality $k$, we obtain a $k$-\dpp. The MH sampler for this process is similar, but a state transition corresponds to swapping two elements (adding $y$ and removing $v$ at the same time). Assume the current set is $Y = Y'\cup\{v\}$. If we propose to delete $v$ and add $y$ to $Y'$, then the corresponding transition probability is 
\begin{align}
q &= \min\Big\{1, {L_{y,y} - L_{y,Y'}L_{Y'}^{-1}L_{Y',y} \over L_{v,v} - L_{v,Y'}L_{Y'}^{-1}L_{Y',v}}\Big\}.
\end{align}
Again, we sample $p\sim\text{Uniform}(0,1)$, but now we must compute two quantities, and hence two sets of lower and upper bounds: $s_i^y$, $t_i^y$ for $L_{y,Y'}L_{Y'}^{-1}L_{Y',y}$ in the $i$-th Gauss quadrature iteration, and $s_j^v$, $t_j^v$ for $L_{v,Y'}L_{Y'}^{-1}L_{Y',v}$ in the $j$-th Gauss quadrature iteration. 
Then if we have $p \le {L_{y,y} - t_i^y\over L_{v,v}- s_j^v}$, we can safely accept the transition; and if $p \ge {L_{y,y} - s_i^y\over L_{v,v} - t_j^v}$ we can safely reject the transition; otherwise, we tighten the bounds via additional Gauss-Radau iterations.

\textbf{Refinements.} We could perform one iteration for both $y$ and $v$, but it may be that one set of bounds is already sufficiently tight, while the other is loose. A straightforward idea would be to judge the tightness of the lower and upper bounds by their difference (gap) $t_i{-}s_i$, and decide accordingly which quadrature to iterate further. 

But the bounds for $y$ and $v$ are not symmetric and contribute differently to the transition decision.
In essence, we need to judge the relation between $p$ and ${L_{y,y} - L_{y,Y'}L_{Y'}^{-1}L_{Y',y} \over L_{v,v} - L_{v,Y'}L_{Y'}^{-1}L_{Y',v}}$, or, equivalently, the relation between $pL_{v,v} - L_{y,y}$ and $p L_{v,Y'}L_{Y'}^{-1}L_{Y',v} - L_{y,Y'}L_{Y'}^{-1}L_{Y',y}$.
Since the left hand side is ``easy'', the essential part is the right hand side. 
Assuming that in practice the impact is larger when the gap is larger, we tighten the bounds for $L_{v,Y}L_Y^{-1}L_{Y,v}$ if $p(t_j^v - s_j^v) > (t_i^y - s_i^y)$, and otherwise tighen the bounds for $L_{y,Y}L_Y^{-1}L_{Y,y}$. Details of the final algorithm with this refinement are shown in Appendix~\ref{append:sec:kdppalgo}.

\subsection{Retrospective Double Greedy Algorithm}
\label{sec:app.greedy}
As indicated in Section~\ref{sec:motiv.app}, a number of applications, including sensing and information maximization with Gaussian Processes, rely on maximizing a submodular function given as $F(S) = \log\det(L_S)$. In general, this function may be non-monotone. In this case, an algorithm of choice is the double greedy algorithm of~\citet{buchbinder12}.

The double greedy algorithm  
starts with two sets $X_0=\emptyset$ and $Y_0=\cy$ and serially iterates through all elements to construct a near-optimal subset. At iteration $i$, it includes element $i$ into $X_{i-1}$ with probability $q_i$, and with probability $1-q_i$ it excludes $i$ from $Y_{i-1}$. The decisive value $q_i$ is determined by the marginal gains $\Delta_i^- = F(Y_{i-1}\backslash\{i\})-F(Y_{i-1})$ and $\Delta_i^+ = F(X_{i-1}\cup\{i\})-F(X_{i-1})$:

\begin{align*}
 q_i = [\Delta_i^+]_+ / [\Delta_i^+]_+ + [\Delta_i^-]_+. 
\end{align*}
For the log-det function, we obtain
\begin{align*}
  \Delta_i^+ &= -\log(L_{i,i} - L_{i,Y_{i-1}'}L_{Y_{i-1}'}^{-1}L_{Y_{i-1}',i})\\
  \Delta_i^- &= \log(L_{i,i} - L_{i,X_{i-1}}L_{X_{i-1}}^{-1}L_{X_{i-1},i}),
\end{align*}
where $Y'_{i-1} = Y_{i-1}\backslash\{i\}$. In other words, at iteration $i$ the algorithm uniformly samples $p\in(0,1)$, and then checks if
\begin{equation*}
  p [\Delta_i^-]_+ \le (1-p)[\Delta_i^+]_+,
\end{equation*}
and if true, adds $i$ to $X_{i-1}$, otherwise removes it from $Y_{i-1}$.

This essential decision, whether to retain or discard an element, again involves bounding BIFs, for which we can take advantage of our framework, and profit from the typical sparsity of the data.  Concretely, we retrospectively compute the lower and upper bounds on these BIFs, i.e., lower and upper bounds $l_i^+$ and $u_i^+$ on $\Delta_i^+$, and $l_i^-$ and $u_i^-$ on $\Delta_i^-$. If $p [u_i^-]_+ \le (1-p)[l_i^+]_+$ we safely add $i$ to $X_{i-1}$; if $p [l_i^-]_+ > (1-p) [u_i^+]_+$ we safely remove $i$ from $Y_{i-1}$; otherwise we compute a set of tighter bounds by further iterating the quadrature.

As before, the bounds for $\Delta_i^-$ and $\Delta_i^+$ may not contribute equally to the transition decision. We can again apply the refinement mentioned in Section~\ref{sec:mcdpp}: if $p([u_i^-]_+ - [l_i^-]_+)\le (1-p)([u_i^+]_+ - [l_i^+]_+)$ we tighten bounds for $\Delta_i^+$, otherwise we tighten bounds for $\Delta_i^-$. The resulting algorithm is shown in Appendix~\ref{append:sec:dgalgo}.

\begin{figure}[!htbp]
\centering
\hspace{-5pt}
	\begin{subfigure}{.31\textwidth}
	\centering
	\includegraphics[width=\textwidth]{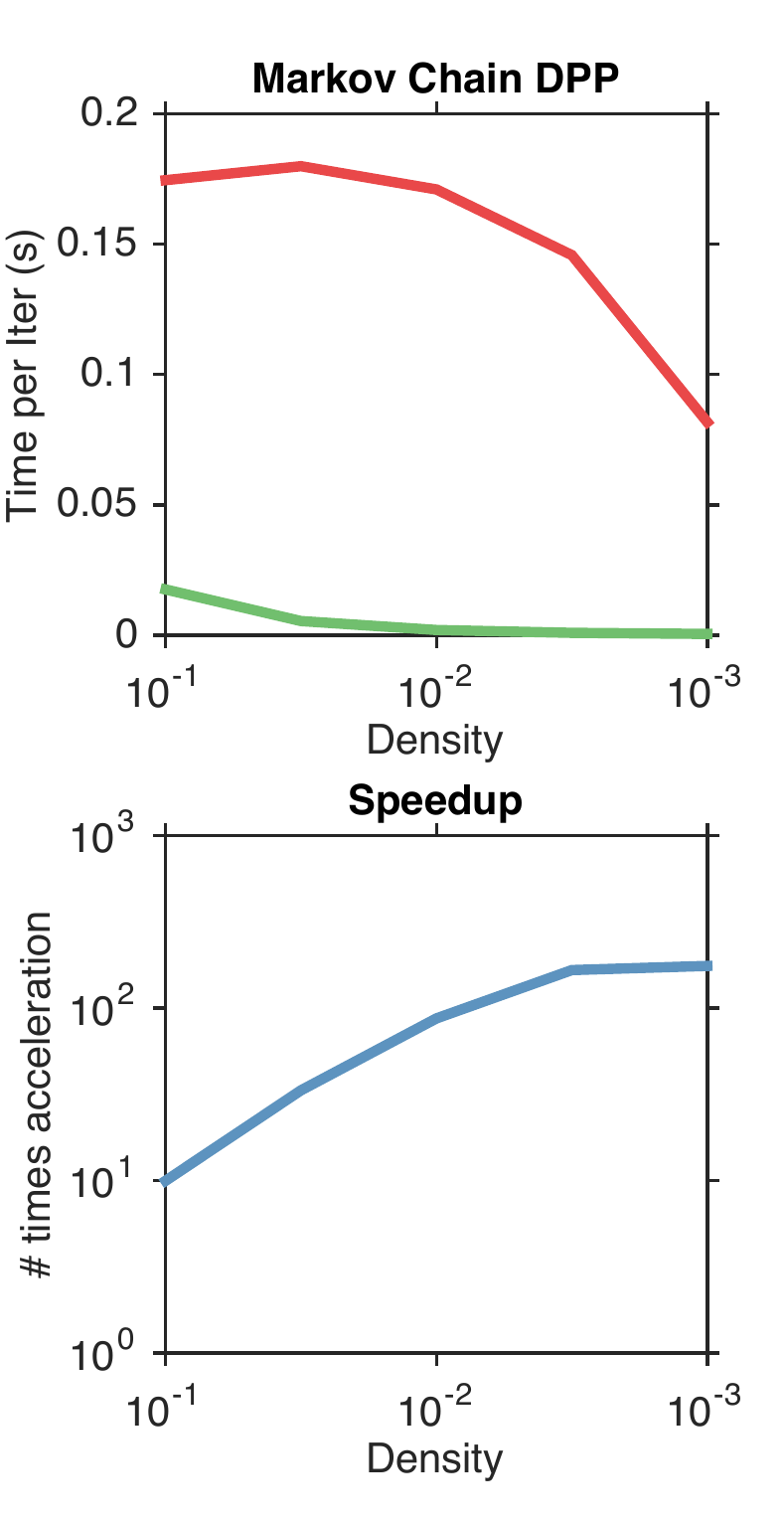}
	\end{subfigure}%
	\begin{subfigure}{.31\textwidth}
	\centering
	\includegraphics[width=\textwidth]{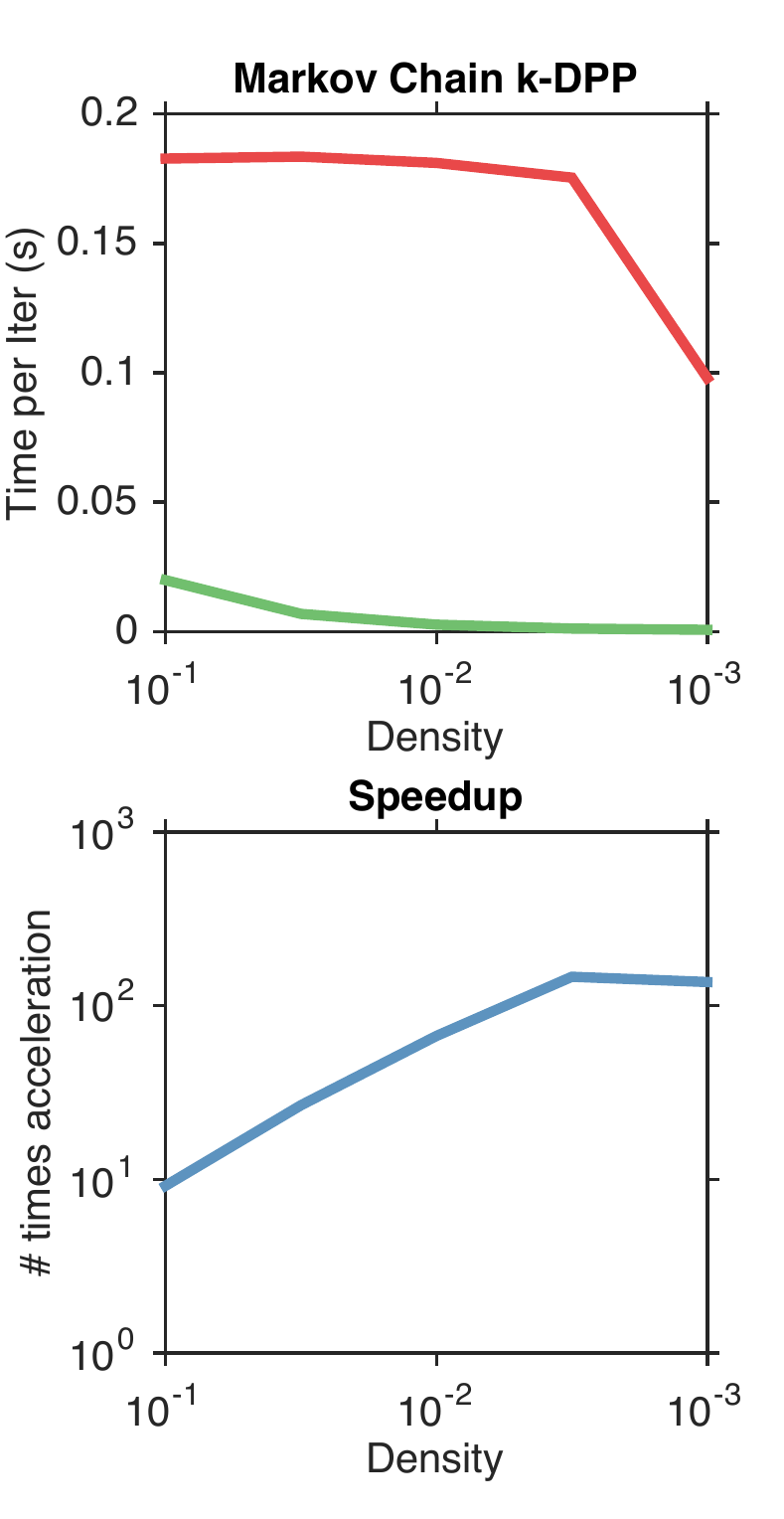}
	\end{subfigure}%
	\begin{subfigure}{.31\textwidth}
	\centering
	\includegraphics[width=\textwidth]{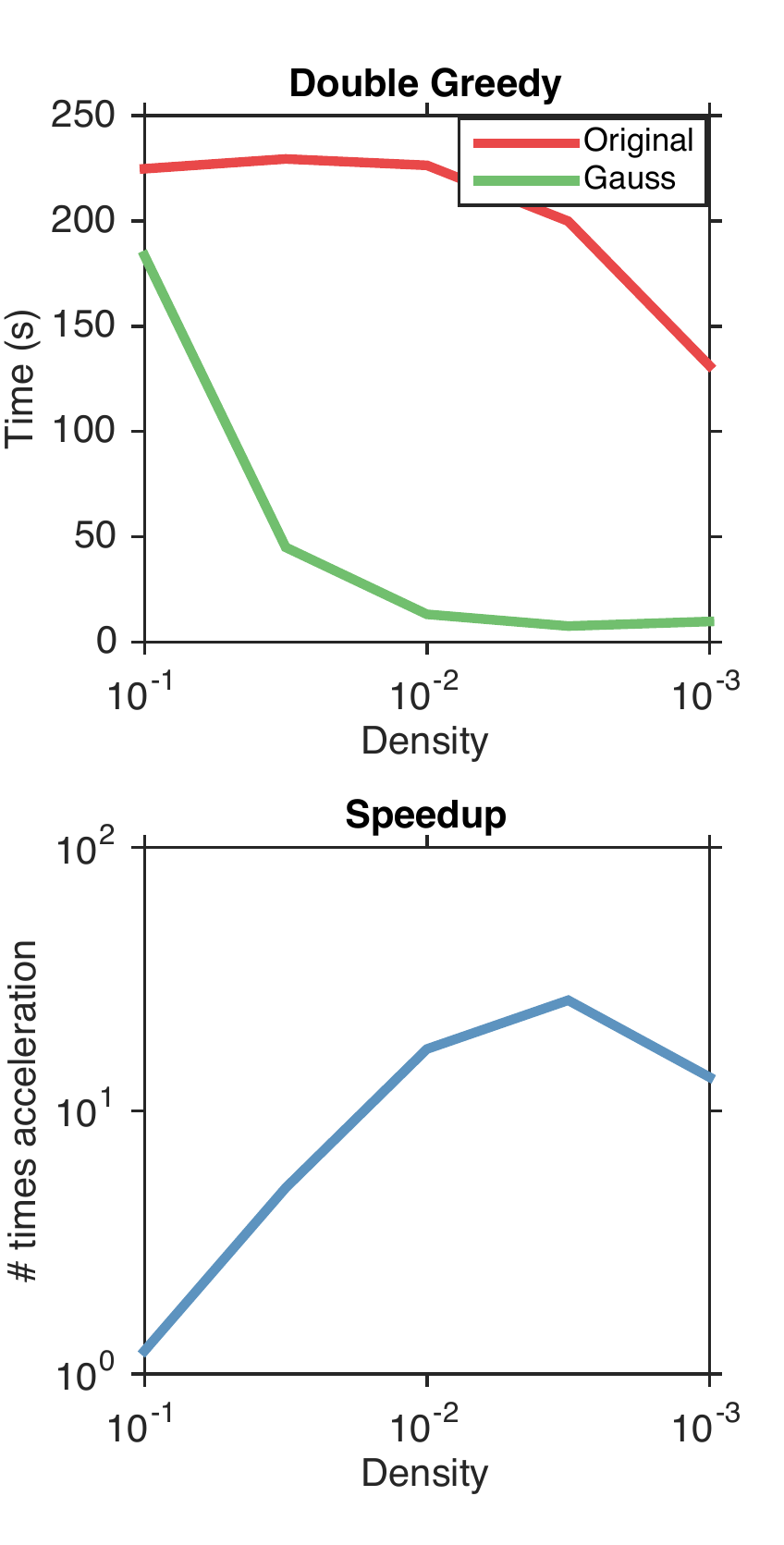}
	\end{subfigure}
	\caption{Running times (top) and corresponding speedup (bottom) on synthetic data. ($k$-)\dpp is initialized with random subsets of size $N/3$ and corresponding running times are averaged over 1,000 iterations of the chain. All results are averaged over 3 runs.}
	\label{fig:syntime}
\end{figure}

\begin{table}[!htbp]
  \begin{center}
  	\begin{tabular}{cccc}
	\toprule
	Data & Dimension & nnz & Density(\%)\\
	\midrule
	Abalone 	& 4,177 	& 144,553 	& 0.83  \\
	Wine		& 4,898		& 2,659,910 & 11.09\\
	\midrule
	GR			& 5,242		& 34,209	& 0.12 \\
	HEP			& 9,877		& 61,821	& 0.0634\\
	\midrule
	Epinions	& 75,879	& 518,231	& 0.009 \\
	Slashdot	& 82,168	& 959,454	& 0.014\\
	\bottomrule
	\end{tabular}
  	\caption{Data. For all datasets we add an 1E-3 times identity matrix to ensure positive definiteness.}
	\label{tab:realdata}
  \end{center}
\end{table}

\begin{table*}[h!]
  \begin{center}
  \begin{small}
    \begin{tabular}{c|cc|cc|cc|cc|cc|cc}
      \toprule
      & \multicolumn{2}{c|}{Abalone} & \multicolumn{2}{c|}{Wine} & \multicolumn{2}{c|}{GR} & \multicolumn{2}{c|}{HEP} & \multicolumn{2}{c|}{Epinions} & \multicolumn{2}{c}{Slashdot} \\    \midrule
      \multirow{2}{*}{\dpp}	& 9.6E-3	& 1x	& 8.5E-2	& 1x	& 9.3E-3	& 1x	& 6.5E-2	& 1x	& 1.46		& 1x		& 5.85		& 1x\\
      						& 5.4E-4	& 17.8x	& 5.9E-3	& 14.4x	& 4.3E-4	& 21.6x	& 5.9E-4	& 110.2x& 3.7E-3	& 394.6x	& 7.1E-3	& 823.9x\\
      \midrule
      \multirow{2}{*}{$k$-\dpp}	& 1.4E-2	& 1x	& 0.15	& 1x		& 1.7E-2	& 1x	& 0.13	& 1x 	& 2.40		& 1x	& 11.83	& 1x\\
      							& 7.3E-4	& 19.2x	& 1.1E-2	& 13.6x		& 7.3E-4	& 23.3x & 9.2E-4 	& 141.3x& 4.9E-3	& 489.8x & 1E-2	& 1183x\\
      \midrule
      \multirow{2}{*}{\dg}	& 1025.6	& 1x	& 1951.3	& 1x	& 965.8	& 1x	& 6269.4	& 1x	& $\ast$ 	& $\ast$	& $\ast$		& $\ast$\\
      						& 17.3		& 59.3x	& 423.2	& 4.6x	&    10		& 9.7x	& 25.3		& 247.8x& 418	& $\ast$	& 712.9	& $\ast$\\
      \bottomrule
    \end{tabular}
    \caption{Running time and speedup for ($k$-)\dpp and double greedy. For results on each dataset (occupying two columns), the first column shows the running time~(in seconds) and the second column shows the speedup. For each algorithm (occupying two rows), the first row shows results from the original algorithm and the second row shows results from algorithms using  our framework. 
    For Epinions and Slashdot, entries of ``$\ast$'' indicate that the experiments did not finish within 24 hours.}
    \label{tab:realspeedup}
    \end{small}
  \end{center}
\end{table*}

\subsection{Empirical Evidence}

We perform experiments on both synthetic and real-world datasets to test the impact of our retrospective quadrature framework in applications.
We focus on ($k$-)\dpp sampling and the double greedy algorithm for the log-det objective.

\subsubsection{Synthetic Datasets}

We generate small sparse matrices using methods similar to Section~\ref{sec:conv}. For ($k$-)\dpp we generate $5000\times 5000$ matrices while for double greedy we use $2000\times 2000$. We vary the density of the matrices from $10^{-3}$ to $10^{-1}$. The running time and speedup are shown in Figure~\ref{fig:syntime}.

The results suggest that our framework greatly accelerates both \dpp sampling and submodular maximization. The speedups are particularly pronounced for sparse matrices.
As the matrices become very sparse, the original algorithms profit from sparsity too, and the difference shrinks a little. Overall, we see that our framework has the potential to lead to substantial speedups for algorithms involving bilinear inverse forms.

\subsubsection{Real Datasets}

We further test our framework on real-world datasets of varying sizes. We selected 6 datasets, four of them are of small/medium size and two are large. 
The four small/medium-sized datasets are used in~\cite{gittens2013revisiting}. 
The first two of small/medium-sized datasets, Abalone and Wine\footnote{Available at \url{http://archive.ics.uci.edu/ml/}.}, are popular datasets for regression, and we construct sparse kernel matrices with an RBF kernel. We set the bandwidth parameter for Abalone as $\sigma = 0.15$ and that for Wine as $\sigma = 1$ and the cut-off parameter as $3\sigma$ for both datasets, as in~\cite{gittens2013revisiting}. 
The other two small/medium-sized datasets are GR~(arXiv High Energy Physics collaboration graph) and HEP~(arXiv General Relativity collaboration graph), where the kernel matrices are Laplacian matrices.
The final two large datasets datasets are Epinions~(Who-trusts-whom network of Epinions) and Slashdot~(Slashdot social network from Feb. 2009)~\footnote{Available at \url{https://snap.stanford.edu/data/}.} with large Laplacian matrices. Dataset statistics are shown in~\reftab{tab:realdata}.

The running times in~\reftab{tab:realspeedup} suggest that the iterative bounds from quadrature significantly accelerate ($k$-)\dpp sampling and double greedy on real data. Our algorithms lead to speedups of up to a thousand times.

On the large sparse matrices, the ``standard'' double greedy algorithm did not finish within 24 hours, due to the expensive matrix operations involved. With our framework, the algorithm needs only 15 minutes.

To our knowledge, these results are the first time to run \dpp and double greedy for information gain on such large datasets.

\subsection{Numerical details}
\label{sec:discuss}
\paragraph{Instability.}
As seen in~\refalgo{algo:gql}, the quadrature algorithm is built upon Lanczos iterations. Although in theory Lanczos iterations construct a set of orthogonal Lanczos vectors, in practice the constructed vectors usually lose orthogonality after some iterations due to rounding errors. One way to deal with this problem  is to reorthogonalize the vectors, either completely at each iteration or selectively~\cite{parlett1979lanczos}. Also, an equivalent Lanczos iteration proposed in~\cite{paige1971computation} which uses a different expression to improve local orthogonality. Further discussion on numerical stability of the method lies beyond the scope of this paper.

\paragraph{Preconditioning.}
For Gauss quadrature on $u^\top A^{-1} u$, 
the convergence rate of bounds is dependent on the condition number of $A$.
We can use preconditioning techniques to get a well-conditioned submatrix and proceed with that.
Concretely, observe that for non-singular $C$,
\begin{align*}
u^\top A^{-1} u &= u^\top C^\top C^{-\top}A^{-1}C^{-1} Cu \\
&= (Cu)(CAC^\top)^{-1}(Cu).
\end{align*}
Thus, if $CAC^\top$ is well-conditioned, we can use it with the vector $Cu$ in Gauss quadrature.

There exists various ways to obtain good preconditioners for an SPD matrix. A simple choice is to use $C = [\text{diag}(A)]^{-1/2}$. 
There also exists methods for efficiently constructing sparse inverse matrix~\cite{benzi1996sparse}.
If $L$ happens to be an SDD matrix, we can use techniques introduced in~\cite{cheng2014scalable} to construct an approximate sparse inverse in near linear time.

\section{Conclusion}

In this paper we present a general and powerful computational framework  for algorithms that rely on computations of bilinear inverse forms. The framework uses Gauss quadrature methods to lazily and iteratively tighten bounds, and is supported by our new theoretical results.
We analyze properties of the various types of Gauss quadratures for approximating the bilinear inverse forms and show that all bounds are monotonically becoming tighter with the number of iterations; those given by Gauss-Radau are superior to those obtained from other Gauss-type quadratures; and both lower and upper bounds enjoy a linear convergence rate. 
We empirically verify the efficiency of our framework and are able to obtain speedups of up to a thousand times for two popular examples: maximizing information gain and sampling from determinantal point processes. 

\paragraph{Acknowledgements} This research was partially supported by NSF CAREER award 1553284 and a Google Research Award.

\bibliographystyle{abbrvnat}
\setlength{\bibsep}{0pt}

\newpage

\begin{appendix}
\onecolumn
\section{Further Background on Gauss Quadrature}
\label{append:sec:gauss}
We present below a more detailed summary of material on Gauss quadrature to make the paper self-contained. 

\subsection{Selecting weights and nodes}
We've described that the Riemann-Stieltjes integral could be expressed as
\begin{equation*}
  I[f] := Q_{n} + R_{n} = \nlsum_{i=1}^n \omega_i f(\theta_i) + \nlsum_{i=1}^m \nu_i f(\tau_i) + R_{n}[f],
\end{equation*}
where $Q_n$ denotes the $n$th degree approximation and $R_n$ denotes a remainder term. 
The weights $\set{\omega_i}_{i=1}^n$, $\set{\nu_i}_{i=1}^m$ and nodes $\set{\theta_i}_{i=1}^n$ are chosen such that for all polynomials of degree less than $2n+m-1$, denoted $f\in \mathbb{P}^{2n + m - 1}$, we have \emph{exact} interpolation $I[f] = Q_{n}$. One way to compute weights and nodes is to set $f(x) = x^i$ for $i\le 2n+m-1$ and then use this exact nonlinear system. But there is an easier way to obtain weights and nodes, namely by using polynomials orthogonal with respect to the measure $\alpha$. Specifically, we construct a sequence of \emph{orthogonal polynomials} $p_0(\lambda), p_1(\lambda),\ldots$ such that $p_i(\lambda)$ is a polynomial in $\lambda$ of degree exactly $k$, and $p_i$, $p_j$ are orthogonal, i.e., they satisfy
\begin{align*}
\int_{\lambda_{\min}}^{\lambda_{\max}} p_i(\lambda) p_j(\lambda) d\alpha(\lambda) = 
\left\{\begin{array}{ll}
1, & i = j\\
0, & \text{otherwise}.
\end{array}
\right.
\end{align*}
The roots of $p_n$ are distinct, real and lie in the interval of $[\lambda_{\min}, \lambda_{\max}]$, and form the nodes $\{\theta_i\}_{i=1}^n$ for Gauss quadrature~(see, e.g.,~\citep[Ch.~6]{golub2009matrices}). 

Consider the two \emph{monic polynomials} whose roots serve as quadrature nodes:
\begin{align*}
\pi_n(\lambda) &= \prod\nolimits_{i=1}^n (\lambda - \theta_i),\quad \rho_m(\lambda) = \prod\nolimits_{i=1}^m(\lambda - \tau_i),
\end{align*}
where $\rho_0 = 1$ for consistency. We further denote $\rho_m^+ = \pm \rho_m$, where the sign is taken to ensure $\rho_m^+\ge 0$ on $[\lambda_{\min},\lambda_{\max}]$.
Then, for $m > 0$, we calculate the quadrature weights as
\begin{align*}
  \omega_i = I\biggl[{\rho_m^+(\lambda) \pi_n(\lambda)\over \rho_m^+(\theta_i)\pi_n'(\theta_i)(\lambda - \theta_i)}\biggr], \quad
  \nu_j = I\biggl[{\rho_m^+(\lambda)\pi_n(\lambda) \over (\rho_m^+)'(\tau_j)\pi_n(\tau_j)(\lambda - \tau_j)}\biggr],
\end{align*}
where $f'(\lambda)$ denotes the derivative of $f$ with respect to $\lambda$.
When $m = 0$ the quadrature degenerates to Gauss quadrature and we have
\begin{align*}
\omega_i &= I\left[{\pi_n(\lambda)\over\pi_n'(\theta_i)(\lambda - \theta_i)}\right].
\end{align*}

Although we have specified how to select nodes and weights for quadrature, these ideas cannot be applied to our problem because the measure $\alpha$ is unknown. Indeed, calculating the measure explicitly would require knowing the entire spectrum of $A$, which is as good as explicitly computing $f(A)$, hence untenable for us. The next section shows how to circumvent the difficulties due to unknown $\alpha$.

\subsection{Gauss Quadrature Lanczos~(GQL)}
The key idea to circumvent our lack of knowledge of $\alpha$ is to recursively construct polynomials called \textit{Lanczos polynomials}. The construction ensures their orthogonality with respect to $\alpha$. Concretely, we construct Lanczos polynomials via the following three-term recurrence:
\begin{equation}
  \label{append:eq:3}
  \begin{split}
    &\beta_i p_i(\lambda) = (\lambda - \alpha_i) p_{i-1}(\lambda) - \beta_{i-1}p_{i-2}(\lambda),\quad i=1,2,\ldots, n\\
    &p_{-1}(\lambda) \equiv 0;\;\;\;p_0(\lambda) \equiv 1,
  \end{split}
\end{equation}
while ensuring $\int_{\lambda_{\min}}^{\lambda_{\max}} d\alpha(\lambda) = 1$. We can express~\eqref{append:eq:3} in matrix form by writing
\begin{align*}
\lambda P_n(\lambda) = J_n P_n(\lambda) + \beta_n p_n(\lambda) e_n,
\end{align*}
where $P_n(\lambda) := [p_0(\lambda),\ldots, p_{n-1}(\lambda)]^\top$, $e_n$ is $n$th canonical unit vector, and $J_n$ is the tridiagonal matrix
\begin{align}
  J_n &= \begin{bmatrix}
    \alpha_1 & \beta_1 & & &\\
    \beta_1 & \alpha_2 & \beta_2 & & \\
    & \beta_2 & \ddots & \ddots &\\
    & & \ddots & \alpha_{n-1} & \beta_{n-1}\\
    & & & \beta_{n-1} & \alpha_n
  \end{bmatrix}.
\end{align}
This matrix is known as the \textit{Jacobi matrix}, and is closed related to Gauss quadrature. The following well-known theorem makes this relation precise.

\begin{theorem}[\cite{wilf1962mathematics,golub1969calculation}]\label{append:thm:gauss}
The eigenvalues of $J_n$ form the nodes $\{\theta_i\}_{i=1}^n$ of Gauss-type quadratures. The weights $\{\omega_i\}_{i=1}^n$ are given by the squares of the first elements of the normalized eigenvectors of $J_n$.
\end{theorem}

Thus, if $J_n$ has the eigendecomposition $J_n = P_n^\top \Gamma P_n$, then for Gauss~quadrature \refthm{append:thm:gauss} yields 
\begin{equation}
  \label{append:eq:4}
Q_{n} = \nlsum_{i=1}^n \omega_i f(\theta_i) = e_1^\top P_n^\top f(\Gamma) P_n e_1 = e_1^\top f(J_n) e_1.
\end{equation}

\paragraph{Specialization.} We now specialize to our main focus, $f(A)=A^{-1}$, for which we prove more precise results. In this case, \eqref{append:eq:4} becomes $Q_n=[J_n^{-1}]_{1,1}$. The task now is to compute $Q_n$, and given $A$, $u$ to obtain the Jacobi matrix $J_n$.

Fortunately, we can efficiently calculate $J_n$ iteratively using the \textit{Lanczos Algorithm}~\cite{lanczos1950iteration}. Suppose we have an estimate $J_i$, in iteration $(i+1)$ of Lanczos, we compute the tridiagonal coefficients $\alpha_{i+1}$ and $\beta_{i+1}$ and add them to this estimate to form $J_{i+1}$. As to $Q_n$, assuming we have already computed $[J_{i}^{-1}]_{1,1}$, letting $j_{i} = J_i^{-1}e_i$ and invoking the Sherman-Morrison identity~\cite{sherman1950adjustment} we obtain the recursion: 
\begin{align}
[J_{i+1}^{-1}]_{1,1} = [J_i^{-1}]_{1,1} + \frac{\beta_i^2 ([j_i]_{1})^2}{\alpha_{i+1} - \beta_i^2 [j_i]_i},
\end{align}
where $[j_i]_1$ and $[j_i]_i$ can be recursively computed using a Cholesky-like factorization of $J_i$~\citep[p.31]{golub2009matrices}. 

For Gauss-Radau quadrature, we need to modify $J_i$ so that it has a prescribed eigenvalue. More precisely, we extend $J_i$ to $\leftgr{J}_i$ for left Gauss-Radau~($\rightgr{J}_i$ for right Gauss-Radau) with $\beta_i$ on the off-diagonal and $\leftgr{\alpha}_i$~($\rightgr{\alpha}_i$) on the diagonal, so that $\leftgr{J}_i$~($\rightgr{J}_i$) has a prescribed eigenvalue of $\lambda_{\min}$~($\lambda_{\max}$). 

For Gauss-Lobatto quadrature, we extend $J_i$ to $\lobatto{J}_i$ with values $\lobatto{\beta}_i$ and $\lobatto{\alpha}_i$ chosen to ensure that $\lobatto{J}_i$ has the prescribed eigenvalues $\lambda_{\min}$ and $\lambda_{\max}$. For more detailed on the construction, see~\cite{golub1973some}. 

For all methods, the approximated values are calculated as $[(J_i')^{-1}]_{1,1}$, where $J_i'\in\{\leftgr{J}_i,\rightgr{J}_i,\lobatto{J}_i\}$ is the modified Jacobi matrix. Here $J_i'$ is constructed at the $i$-th iteration of the algorithm.

The algorithm for computing Gauss, Gauss-Radau, and Gauss-Lobatto quadrature rules with the help of Lanczos iteration is called \emph{Gauss Quadrature Lanczos} (GQL) and is shown in~\cite{golub1997matrices}. We recall its pseudocode in~\refalgo{algo:gql} to make our presentation self-contained (and for our proofs in Section~\ref{sec:main}).

\begin{algorithm}[t]
	\caption{Gauss Quadrature Lanczos (GQL)}\label{append:algo:gql}
	\begin{algorithmic} 
	\REQUIRE{$u$ and $A$ the corresponding vector and matrix, $\lambda_{\min}$ and $\lambda_{\max}$ lower and upper bounds for the spectrum of $A$}
	\ENSURE{$g_i$, $\rightgr{g}_i$, $\leftgr{g}_i$ and $\lobatto{g}_i$ the Gauss, right Gauss-Radau, left Gauss-Radau and Gauss-Lobatto quadrature computed at $i$-th iteration}
	\STATE \textbf{Initialize}: $u_{-1} = 0$, $u_0 = u / \|u\|$, $\alpha_1 = u_0^\top A u_0$, $\beta_1 = \|(A - \alpha_1 I)u_0\|$, $g_1 = \|u\|/\alpha_1$, $c_1 = 1$, $\delta_1 = \alpha_1$, $\leftgr{\delta}_1 = \alpha_1 - \lambda_{\min}$, $\rightgr{\delta}_1 = \alpha_1 - \lambda_{\max}$, $u_1 = (A - \alpha_1 I) u_0 / \beta_1$, $i = 2$
	\WHILE{$i \le N$}
		\STATE $\alpha_i = u_{i-1}^\top Au_{i-1}$ \COMMENT{Lanczos Iteration}
		\STATE $\tilde{u}_i = Au_{i-1} - \alpha_i u_{i-1} - \beta_{i-1}u_{i-2}$
		\STATE $\beta_i = \|\tilde{u}_i\|$
		\STATE $u_i = \tilde{u}_i / \beta_i$
		\STATE $g_i = g_{i-1} + {\|u\|\beta_{i-1}^2c_{i-1}^2\over \delta_{i-1}(\alpha_i\delta_{i-1} - \beta_{i-1}^2)}$ \COMMENT{Update $g_i$ with Sherman-Morrison formula}
		\STATE $c_i = {c_{i-1}\beta_{i-1}/ \delta_{i-1}}$ 
		\STATE $\delta_i = \alpha_i - {\beta_{i-1}^2\over \delta_{i-1}}$, $\leftgr{\delta}_i = \alpha_i - \lambda_{\min} - {\beta_{i-1}^2\over \leftgr{\delta}_{i-1}}$, $\rightgr{\delta}_i = \alpha_i - \lambda_{\max} - {\beta_{i-1}^2\over \rightgr{\delta}_{i-1}}$ 
		\STATE $\leftgr{\alpha}_i = \lambda_{\min} + {\beta_i^2\over \leftgr{\delta}_i}$, $\rightgr{\alpha}_i = \lambda_{\max} + {\beta_i^2\over \rightgr{\delta}_i}$ \COMMENT{Solve for $\leftgr{J}_i$ and $\rightgr{J}_i$}
		\STATE $\lobatto{\alpha}_i = {\leftgr{\delta}_i\rightgr{\delta}_i\over \rightgr{\delta}_i - \leftgr{\delta}_i}({\lambda_{\max}\over \leftgr{\delta}_i} - {\lambda_{\min}\over \rightgr{\delta}_i})$, $(\lobatto{\beta}_i)^2 = {\leftgr{\delta}_i \rightgr{\delta}_i\over \rightgr{\delta}_i - \leftgr{\delta}_i}(\lambda_{\max} - \lambda_{\min})$ \COMMENT{Solve for $\lobatto{J}_i$}
		\STATE $\leftgr{g}_i = g_i + {\beta_i^2 c_i^2\|u\|\over \delta_i(\leftgr{\alpha}_i\delta_i - \beta_i^2)}$, $\rightgr{g}_i = g_i + {\beta_i^2 c_i^2\|u\|\over \delta_i(\rightgr{\alpha}_i \delta_i - \beta_i^2)}$, $\lobatto{g}_i = g_i + {(\lobatto{\beta}_i)^2 c_i^2\|u\|\over \delta_i(\lobatto{\alpha}_i\delta_i - (\lobatto{\beta}_i)^2)}$ \COMMENT{Update $\rightgr{g}_i$, $\leftgr{g}_i$ and $\lobatto{g}_i$ with Sherman-Morrison formula}
		\STATE  $i = i + 1$
	\ENDWHILE
\end{algorithmic}
\end{algorithm}

The error of approximating $I[f]$ by Gauss-type quadratures can be expressed as
\begin{equation*}
R_n[f] = {f^{(2n+m)}(\xi)\over (2n+m)!} I[\rho_m\pi_n^2],
\end{equation*}
for some $\xi\in[\lambda_{\min},\lambda_{\max}]$~(see, e.g.,~\cite{stoer2013introduction}). Note that $\rho_m$ does not change sign in $[\lambda_{\min},\lambda_{\max}]$; but with different values of $m$ and $\tau_j$ we obtain different (but fixed) signs for $R_n[f]$ using $f(\lambda) = 1/\lambda$ and $\lambda_{\min} > 0$. Concretely, for Gauss quadrature $m = 0$ and $R_n[f] \ge 0$; for left Gauss-Radau $m = 1$ and  $\tau_1 = \lambda_{\min}$, so we have $R_n[f] \le 0$; for right Gauss-Radau we have $m = 1$ and $\tau_1 = \lambda_{\max}$, thus $R_n[f] \ge 0$; while for Gauss-Lobatto we have $m = 2$, $\tau_1 = \lambda_{\min}$ and $\tau_2 = \lambda_{\max}$, so that $R_n[f] \le 0$. This behavior of the errors clearly shows the ordering relations between the target values and the approximations made by the different quadrature rules. Lemma~\ref{lem:bounds} (see e.g., \cite{meurant1997computation}) makes this claim precise.

\begin{lemma}\label{append:lem:bounds}
  Let $g_i$, $\leftgr{g}_i$, $\rightgr{g}_i$, and $\lobatto{g}_i$ be the approximations at the $i$-th iteration of Gauss, left Gauss-Radau, right Gauss-Radau, and Gauss-Lobatto quadrature, respectively. Then, $g_i$ and $\rightgr{g}_i$ provide lower bounds on $u^\top A^{-1} u$, while $\leftgr{g}_i$ and $\lobatto{g}_i$ provide upper bounds.
\end{lemma}

The final connection we recall as background is the method of conjugate gradients. This helps us analyze the speed at which quadrature converges to the true value (assuming exact arithmetic).

\subsection{Relation with Conjugate Gradient}
\label{sec:cg}
While Gauss-type quadratures relate to the Lanczos algorithm, Lanczos itself is closely related to conjugate gradient (CG)~\cite{hestenes1952methods}, a well-known method for solving $Ax = b$ for positive definite $A$.

We recap this connection below. Let $x_k$ be the estimated solution at the $k$-th CG iteration. If $x^*$ denotes the true solution to $Ax = b$, then the \emph{error} $\eps_k$ and \emph{residual} $r_k$ are defined as
\begin{align}
\label{append:eq:5}
\eps_k := x^* - x_k,\qquad r_k = A\eps_k = b - Ax_k,
\end{align}
At the $k$-th iteration, $x_k$ is chosen such that $r_k$ is orthogonal to the $k$-th \emph{Krylov space}, i.e., the linear space $\mathcal{K}_k$ spanned by $\{r_0, Ar_0, \ldots, A^{k-1}r_0\}$. It can be shown~\cite{rard2006lanczos} that $r_k$ is a scaled Lanczos vector from the $k$-th iteration of Lanczos  started with $r_0$. Noting the relation between Lanczos and Gauss quadrature applied to appoximate $r_0^\top A^{-1} r_0$, one obtains the following theorem that relates CG with GQL.

\begin{theorem}[CG and GQL; \citep{meurant1999numerical}]\label{append:thm:relation}
Let $\eps_k$ be the error as in~\eqref{append:eq:5}, and let $\|\eps_k\|_A^2 := \eps_k^TA\eps_k$. Then, it holds that
\begin{align*}
  \|\eps_k\|_A^2 = \|r_0\|^2 ([J_N^{-1}]_{1,1} - [J_k^{-1}]_{1,1}),
\end{align*}
where $J_k$ is the Jacobi matrix at the $k$-th Lanczos iteration starting with $r_0$.
\end{theorem}
Finally, the rate at which $\|\eps_k\|_A^2$ shrinks has also been well-studied, as noted below.
\begin{theorem}[CG rate, see e.g.~\cite{shewchuk1994introduction}]\label{append:thm:cgconv}
Let $\eps_k$ be the error made by CG at iteration $k$ when started with $x_0$. Let $\kappa$ be the condition number of $A$, i.e., $\kappa = \lambda_1/\lambda_N$. Then, the error norm at iteration $k$ satisfies
\begin{align*}
\|\eps_k\|_A \le 2\Bigl({\sqrt{\kappa} - 1\over \sqrt{\kappa} + 1}\Bigr)^k\|\eps_0\|_A.
\end{align*}
\end{theorem}

Due to these explicit relations between CG and Lanczos, as well as between Lanczos and Gauss quadrature, we readily obtain the following convergence rate for relative error of Gauss quadrature.
\begin{theorem}[Gauss quadrature rate]\label{append:thm:gaussconv}
The $i$-th iterate of Gauss quadrature satisfies the relative error bound
\begin{align*}
  \frac{g_N - g_i}{g_N} \le 2\Bigl({\sqrt{\kappa} - 1\over \sqrt{\kappa} + 1}\Bigr)^i.
\end{align*}
\begin{proof}
  This is obtained by exploiting relations among CG, Lanczos and Gauss quadrature. Set $x_0 = 0$ and $b = u$. Then, $\eps_0 = x^*$ and $r_0 = u$. An application of~\refthm{append:thm:relation} and \refthm{append:thm:cgconv} thus yields the bound
\begin{align*}
\|\eps_i\|_A^2 &=\|u\|^2 ([J_N^{-1}]_{1,1} - [J_i^{-1}]_{1,1}) = g_N - g_i\\
&\le\quad2\Bigl(\frac{\sqrt{\kappa} - 1}{\sqrt{\kappa} + 1}\Bigr)^i\|\eps_0\|_A
= 2\Bigl(\frac{\sqrt{\kappa} - 1}{\sqrt{\kappa} + 1}\Bigr)^i u^\top A^{-1} u
= 2\Bigl(\frac{\sqrt{\kappa} - 1}{\sqrt{\kappa} + 1}\Bigr)^i g_N
\end{align*}
where the last equality draws from~\reflem{append:lem:exact}.
\end{proof}
\end{theorem}
\noindent In other words, \refthm{append:thm:gaussconv} shows that the iterates of Gauss quadrature converge linearly.

\section{Proofs for Main Theoretical Results}
\label{app:sec:proofs}
We begin by proving an exactness property of Gauss and Gauss-Radau quadrature.
\begin{lemma}[Exactness]
\label{append:lem:exact}
With $A$ being symmetric positive definite with simple eigenvalues, the iterates $g_N$, $\leftgr{g}_N$, and $\rightgr{g}_N$ are exact. Namely, after $N$ iterations they satisfy
\begin{align*}
  g_N = \leftgr{g}_N = \rightgr{g}_N = u^\top A^{-1} u.
\end{align*}
\begin{proof}
Observe that the Jacobi tridiagonal matrix can be computed via Lanczos iteration, and Lanczos is essentially essentially an iterative tridiagonalization of $A$. At the $i$-th iteration we have $J_i = V_i^\top A V_i$,
where $V_i \in\mathbb{R}^{N\times i}$ are the first $i$ Lanczos vectors~(i.e., a basis for the $i$-th Krylov space). Thus, $J_N = V_N^\top AV_N$ where $V_N$ is an $N\times N$ orthonormal matrix, showing that $J_N$ has the same eigenvalues as $A$. As a result $\pi_N(\lambda) = \prod_{i=1}^N (\lambda - \lambda_i)$, and it follows that the remainder
\begin{equation*}
R_N[f] = {f^{(2N)}(\xi)\over (2N)!} I[\pi_N^2] = 0,
\end{equation*}
for some scalar $\xi\in[\lambda_{\min},\lambda_{\max}]$, which shows that $g_N$ is exact for $u^\top A^{-1} u$.
For left and right Gauss-Radau quadrature, we have $\beta_N = 0$, $\leftgr{\alpha}_N = \lambda_{\min}$, and $\rightgr{\alpha}_N = \lambda_{\max}$, while all other elements of the $(N+1)$-th row or column of $J_{N}'$ are zeros. Thus, the eigenvalues of $J_{N}'$ are $\lambda_1,\ldots, \lambda_N, \tau_1$, and $\pi_N(\lambda)$ again equals $\prod_{i=1}^N(\lambda - \lambda_i)$. As a result, the remainder satisfies
\begin{equation*}
R_N[f] = {f^{(2N)}(\xi)\over (2N)!} I[(\lambda - \tau_1)\pi_N^2] = 0,
\end{equation*}
from which it follows that both $\rightgr{g}_N$ and $\leftgr{g}_N$ are exact.
\end{proof}
\end{lemma}

The convergence rate in~\refthm{append:thm:cgconv} and the final exactness of iterations in~\reflem{append:lem:exact} does not necessarily indicate that we are making progress at each iterations. However, by exploiting the relations to CG we can indeed conclude that we are making progress in each iteration in Gauss quadrature.
\begin{theorem}\label{append:thm:monogauss}
The approximation $g_i$ generated by Gauss quadrature is monotonically nondecreasing, i.e.,
\begin{align*}
  g_i\le g_{i+1},\quad\text{for}\ i < N.
\end{align*}
\begin{proof}
At each iteration $r_i$ is taken to be orthogonal to the $i$-th Krylov space: $\ck_i = \textrm{span}\{u, Au, \ldots, A^{i-1}u\}$. Let $\Pi_i$ be the projection onto 
the complement space of $\ck_i$.  The residual then satisfies
\begin{align*}
\|\eps_{i+1}\|_A^2 &= \eps_{i+1}^TA\eps_{i+1} = r_{i+1}^\top A^{-1} r_{i+1}\\
&= (\Pi_{i+1}r_i)^\top A^{-1}\Pi_{i+1}r_i\\
&= r_i^\top (\Pi_{i+1}^\top A^{-1}\Pi_{i+1}) r_i \le r_i A^{-1} r_i,
\end{align*} 
where the last inequality follows from $\Pi_{i+1}^\top A^{-1}\Pi_{i+1}\preceq A^{-1}$. Thus $\|\eps_i\|_A^2$ is monotonically nonincreasing, whereby $g_N - g_i\ge 0$ is monotonically decreasing and thus $g_i$ is monotonically nondecreasing.
\end{proof}
\end{theorem}

Before we proceed to Gauss-Radau, let us recall a useful theorem and its corollary.
\begin{theorem}[Lanczos Polynomial~\cite{golub2009matrices}]\label{append:thm:lanczospoly}
Let $u_i$ be the vector generated by~\refalgo{algo:gql} at the $i$-th iteration; let $p_i$ be the Lanczos polynomial of degree $i$. Then we have
\begin{equation*}
u_i = p_i(A)u_0,\quad\text{where}\ p_i(\lambda) = (-1)^i {\det(J_i - \lambda I)\over \prod_{j=1}^i \beta_j}.
\end{equation*}
\end{theorem}
From the expression of Lanczos polynomial we have the following corollary specifying the sign of the polynomial at specific points.
\begin{corollary}
Assume $i<N$. If $i$ is odd, then $p_i(\lambda_{\min}) < 0$; for even $i$, $p_i(\lambda_{\min})> 0$, while $p_i(\lambda_{\max})> 0$ for any $i < N$.
\begin{proof}
Since $J_i = V_i^\top A V_i$ is similar to $A$, its spectrum is bounded by $\lambda_{\min}$ and $\lambda_{\max}$ from left and right. Thus, $J_i - \lambda_{\min}$ is positive semi-definite, and $J_i - \lambda_{\max}$ is negative semi-definite. Taking $(-1)^i$ into consideration we will get the desired conclusions. 
\end{proof}
\end{corollary}

We are ready to state our main result that compares (right) Gauss-Radau with Gauss quadrature.
\begin{theorem}[\refthm{thm:lowbtwn} in the main text]\label{append:thm:lowbtwn}
  Let $i < N$. Then, $\rightgr{g}_i$ gives better bounds than $g_i$ but worse bounds than $g_{i+1}$; more precisely,
\begin{equation}
\label{append:eq:6}
g_i\le \rightgr{g}_i\le g_{i+1},\;\;i < N.
\end{equation}
\begin{proof}
  We prove inequality~\eqref{append:eq:6} using the recurrences satisfied by $g_i$ and $\rightgr{g}_i$ (see Alg.~\ref{algo:gql})

  \noindent\emph{Upper bound: $\rightgr{g}_i\le g_{i+1}$.} The iterative quadrature algorithm uses the recursive updates 
  \begin{align*}
    \rightgr{g}_i &= g_i + {\beta_i^2 c_i^2\over \delta_i(\rightgr{\alpha}_i \delta_i - \beta_i^2)},\\
    \quad g_{i+1} &= g_i + {\beta_{i}^2c_i^2\over \delta_i(\alpha_{i+1}\delta_i - \beta_{i}^2)}.
  \end{align*}
  It suffices to thus compare $\rightgr{\alpha}_i$ and $\alpha_{i+1}$. The three-term recursion for Lanczos polynomials shows that
  \begin{align*}
    \beta_{i+1}p_{i+1}(\lambda_{\max}) &= (\lambda_{\max} - \alpha_{i+1})p_i(\lambda_{\max}) - \beta_i p_{i-1}(\lambda_{\max}) > 0,\\
    \beta_{i+1}p_{i+1}^*(\lambda_{\max}) &= (\lambda_{\max} - \rightgr{\alpha}_i)p_i(\lambda_{\max}) - \beta_i p_{i-1}(\lambda_{\max}) = 0,
  \end{align*}
  where $p_{i+1}$ is the original Lanczos polynomial, and $p^*_{i+1}$ is the modified polynomial that has $\lambda_{\max}$ as a root. Noting that $p_i(\lambda_{\max}) > 0$, we see that $\alpha_{i+1} \le \rightgr{\alpha}_i$.  Moreover, from~\refthm{append:thm:monogauss} we know that the $g_i$'s are monotonically increasing, whereby $\delta_i(\alpha_{i+1}\delta_i - \beta_i^2) > 0$. It follows that
  \begin{align*}
    0 < \delta_i(\alpha_{i+1}\delta_i - \beta_i^2) \le \delta_i(\rightgr{\alpha}_i \delta_i - \beta_i^2),
  \end{align*}
  and from this inequality it is clear that $\rightgr{g}_i\le g_{i+1}$.

  \noindent\emph{Lower-bound: $g_i\le \rightgr{g}_i$.} Since $\beta_i^2 c_i^2\ge 0$ and $\delta_i(\rightgr{\alpha}_i\delta_i - \beta_i^2) \ge \delta_i(\alpha_{i+1}\delta_i - \beta_i^2) >0$, we readily obtain
  \begin{equation*}
    g_i\le g_i + {\beta_i^2 c_i^2\over \delta_i(\rightgr{\alpha}_i\delta_i - \beta_i^2)} = \rightgr{g}_i.\qedhere
  \end{equation*}
\end{proof}
\end{theorem}

Combining~\refthm{append:thm:lowbtwn} with the convergence rate of relative error for Gauss quadrature (\refthm{append:thm:gaussconv}) immediately yields the following convergence rate for right Gauss-Radau quadrature:
\begin{theorem}[Relative error of right Gauss-Radau, \refthm{thm:rrconv} in the main text]
\label{append:thm:rrconv}
  For each $i$, the right Gauss-Radau $\rightgr{g}_i$ iterates satisfy
\begin{equation*}
  \frac{g_N - \rightgr{g}_i}{g_N} \le 2\Bigl(\frac{\sqrt{\kappa} - 1}{\sqrt{\kappa} + 1}\Bigr)^i.
\end{equation*}

\end{theorem}
This results shows that with the same number of iterations,  right Gauss-Radau gives superior approximation over Gauss quadrature, though they share the same relative error convergence rate.

\vskip5pt
Our second main result compares Gauss-Lobatto with (left) Gauss-Radau quadrature.\vspace*{-5pt}
\begin{theorem}[\refthm{thm:upbtwn} in the main text]\label{append:thm:upbtwn}
  Let $i < N$. Then, $\leftgr{g}_i$ gives better upper bounds than $\lobatto{g}_{i}$ but worse than $\lobatto{g}_{i+1}$; more precisely,
  \begin{equation*}
    \lobatto{g}_{i+1} \le \leftgr{g}_i \le \lobatto{g}_i,\quad i < N.
  \end{equation*}
\begin{proof}
We prove these inequalities using the recurrences for $\leftgr{g}_i$ and $\lobatto{g}_i$ from~\refalgo{append:algo:gql}. 

\noindent\emph{$\leftgr{g}_i \le \lobatto{g}_i$}: From~\refalgo{append:algo:gql} we observe that $\lobatto{\alpha}_i = \lambda_{\min} + {(\lobatto{\beta}_i)^2\over \leftgr{\delta}_i}$.
Thus we can write $\leftgr{g}_i$ and $\lobatto{g}_i$ as
\begin{align*}
\leftgr{g}_i &= g_i + {\beta_i^2 c_i^2\over \delta_i(\leftgr{\alpha}_i\delta_i - \beta_i^2)} 
= g_i + {\beta_i^2 c_i^2\over \lambda_{\min} \delta_i^2 + \beta_i^2({\delta_i^2/ \leftgr{\delta}_i} - \delta_i)}\\
\lobatto{g}_i &= g_i + {(\lobatto{\beta}_i)^2 c_i^2\over \delta_i(\lobatto{\alpha}_i\delta_i - (\lobatto{\beta}_i)^2)} 
= g_i + {(\lobatto{\beta}_i)^2 c_i^2\over \lambda_{\min} \delta_i^2 + (\lobatto{\beta}_i)^2({\delta_i^2/ \leftgr{\delta}_i} - \delta_i)}
\end{align*}
To compare these quantities, as before it is helpful to begin with the original three-term recursion for the Lanczos polynomial, namely
\begin{equation*}
\beta_{i+1} p_{i+1}(\lambda) = (\lambda - \alpha_{i+1}) p_i(\lambda) - \beta_i p_{i-1}(\lambda).
\end{equation*}
In the construction of Gauss-Lobatto, to make a new polynomial of order $i+1$ that has roots $\lambda_{\min}$ and $\lambda_{\max}$, we add $\sigma_1 p_i(\lambda)$ and $\sigma_2 p_{i-1}(\lambda)$ to the original polynomial to ensure
\begin{align*}
\left\{\begin{array}{ll}
\beta_{i+1}p_{i+1}(\lambda_{\min}) + \sigma_1 p_i(\lambda_{\min}) + \sigma_2 p_{i-1}(\lambda_{\min}) &= 0,\\
\beta_{i+1}p_{i+1}(\lambda_{\max}) + \sigma_1 p_i(\lambda_{\max}) + \sigma_2 p_{i-1}(\lambda_{\max}) &= 0.
\end{array}\right.
\end{align*}
Since $\beta_{i+1}$, $p_{i+1}(\lambda_{\max})$, $p_i(\lambda_{\max})$ and $p_{i-1}(\lambda_{\max})$ are all greater than $0$, $\sigma_1 p_i(\lambda_{\max}) + \sigma_2 p_{i-1}(\lambda_{\max}) < 0$. To determine the sign of polynomials at $\lambda_{\min}$, consider the two cases:
\begin{enumerate}
  \setlength{\itemsep}{1pt}
\item Odd $i$. In this case $p_{i+1}(\lambda_{\min}) > 0$, $p_i(\lambda_{\min}) < 0$, and $p_{i-1}(\lambda_{\min}) > 0$;
\item Even $i$. In this case $p_{i+1}(\lambda_{\min}) < 0$, $p_i(\lambda_{\min}) > 0$, and $p_{i-1}(\lambda_{\min}) < 0$.
\end{enumerate}
Thus, if $S = (\sgn (\sigma_1),\sgn (\sigma_2))$, where the signs take values in $\{0,\pm 1\}$, then $S\ne (1,1)$, $S\ne (-1,1)$ and $S\ne (0,1)$. Hence, $\sigma_2 \le 0$ must hold, and thus $(\lobatto{\beta}_i)^2 = (\beta_i - \sigma_2)^2 \ge\beta_i^2$ given that $\beta_i^2 > 0$ for $i < N$. 

Using $(\lobatto{\beta}_i)^2 \ge \beta_i^2$ with $\lambda_{\min} c_i^2(\delta_i)^2 \ge 0$, an application of monotonicity of the univariate function $g(x) = {ax\over b + cx}$ for $ab \ge 0$ 
to the recurrences defining $\leftgr{g}_i$ and $\lobatto{g}_i$ yields the desired inequality $\leftgr{g}_i\le \lobatto{g}_i$.

\noindent\emph{$\lobatto{g}_{i+1} \le \leftgr{g}_i$}: From recursion formulas we have
\begin{align*}
\leftgr{g}_i &= g_i + {\beta_i^2c_i^2 \over \delta_i (\leftgr{\alpha}_i \delta_i - \beta_i^2)},\\
\lobatto{g}_{i+1} &= g_{i+1} + {(\lobatto{\beta}_{i+1})^2 c_{i+1}^2\over \delta_{i+1}(\lobatto{\alpha}_{i+1}\delta_{i+1} - (\lobatto{\beta}_{i+1})^2)}.
\end{align*}
Establishing $\leftgr{g}_i\ge \lobatto{g}_{i+1}$ thus amounts to showing that (noting the relations among $g_i$, $\leftgr{g}_i$ and $\lobatto{g}_i$):
\begin{align*}
&\frac{\beta_i^2c_i^2}{\delta_i (\leftgr{\alpha}_i \delta_i - \beta_i^2)} - \frac{\beta_i^2c_i^2}{\delta_i (\alpha_{i+1} \delta_i - \beta_i^2)}\quad\ge\quad
\frac{(\lobatto{\beta}_{i+1})^2 c_{i+1}^2} {\delta_{i+1}(\lobatto{\alpha}_{i+1}\delta_{i+1} - (\lobatto{\beta}_{i+1})^2)}\\
\Longleftrightarrow\quad& {\beta_i^2c_i^2 \over \delta_i (\leftgr{\alpha}_i \delta_i - \beta_i^2)} - {\beta_i^2c_i^2 \over \delta_i (\alpha_{i+1} \delta_i - \beta_i^2)}\quad\ge\quad\frac{(\lobatto{\beta}_{i+1})^2 c_{i}^2 \beta_i^2}{(\delta_i)^2\delta_{i+1}(\lobatto{\alpha}_{i+1}\delta_{i+1} - (\lobatto{\beta}_{i+1})^2)}\\
\Longleftrightarrow\quad& {1\over \leftgr{\alpha}_i\delta_i - \beta_i^2} - {1\over \alpha_{i+1}\delta_i - \beta_i^2}\quad\ge\quad{(\lobatto{\beta}_{i+1})^2 \over \delta_i\delta_{i+1}(\lobatto{\alpha}_{i+1}\delta_{i+1} - (\lobatto{\beta}_{i+1})^2)}\\
\Longleftrightarrow\quad&{1\over (\alpha_{i+1} - \leftgr{\delta}_{i+1}) - {\beta_i^2/ \delta_i}} - {1\over \alpha_{i+1} - {\beta_i^2 /\delta_i}}\quad\ge\quad
{1\over \delta_{i+1} ({\lobatto{\alpha}_{i+1}\delta_{i+1}/ (\lobatto{\beta}_{i+1})^2} - 1)}\;\;\;\;(\text{\reflem{append:lem:delta}})\\
\Longleftrightarrow\quad& {1\over \delta_{i+1} - \leftgr{\delta}_{i+1}} - {1\over \delta_{i+1}}\quad\ge\quad{1\over \delta_{i+1} ({\lambda_{\min}\delta_{i+1}\over (\lobatto{\beta}_{i+1})^2} + {\delta_{i+1}\over \leftgr{\delta}_{i+1}} - 1)}\\
\Longleftrightarrow\quad&{\lambda_{\min}\delta_{i+1}\over (\lobatto{\beta}_{i+1})^2} + {\delta_{i+1}\over \leftgr{\delta}_{i+1}} - 1 \ge {\delta_{i+1}\over \leftgr{\delta}_{i+1}} - 1\\
\Longleftrightarrow\quad& {\lambda_{\min}\delta_{i+1}\over (\lobatto{\beta}_{i+1})^2} \ge 0,
\end{align*}
where the last inequality is obviously true; hence the proof is complete.
\end{proof}
\end{theorem}

In summary, we have the following corollary for all the four quadrature rules:

\begin{corollary}[Monotonicity of Lower and Upper Bounds, \refcor{cor:monlow} in the main text]\label{append:cor:monlow}
As the iteration proceeds, $g_i$ and $\rightgr{g}_i$ gives increasingly better asymptotic lower bounds and $\leftgr{g}_i$ and $\lobatto{g}_i$ gives increasingly better upper bounds, namely
\begin{align*}
  g_i \le g_{i+1}; \quad&\rightgr{g}_i \le \rightgr{g}_{i+1}\\
  \leftgr{g}_i \ge \leftgr{g}_{i+1}; \quad& \lobatto{g}_i \ge \lobatto{g}_{i+1}.
\end{align*}
\begin{proof} 
Directly drawn from~\refthm{append:thm:monogauss},~\refthm{append:thm:lowbtwn} and~\refthm{append:thm:upbtwn}.
\end{proof}
\end{corollary}

Before proceeding further to our analysis of convergence rates of left Gauss-Radau and Gauss-Lobatto, we note two technical results that we will need.
\begin{lemma}\label{append:lem:delta}
  Let $\alpha_{i+1}$ and $\leftgr{\alpha}_i$ be as in Alg.~\ref{algo:gql}. The difference $\Delta_{i+1} = \alpha_{i+1} - \leftgr{\alpha}_i$ satisfies $\Delta_{i+1} = \leftgr{\delta}_{i+1}$.
\begin{proof}
  From the Lanczos polynomials in the definition of left Gauss-Radau quadrature we have 
  \begin{align*}
    \beta_{i+1}&p_{i+1}^*(\lambda_{\min})  = \bigl(\lambda_{\min} - \leftgr{\alpha}_i\bigr)p_i(\lambda_{\min}) - \beta_i p_{i-1}(\lambda_{\min})\\
     &= \bigl(\lambda_{\min} - (\alpha_{i+1} - \Delta_{i+1})\bigr)p_i(\lambda_{\min}) - \beta_i p_{i-1}(\lambda_{\min}) \\
     &= \beta_{i+1}p_{i+1}(\lambda_{\min}) + \Delta_{i+1}p_i(\lambda_{\min}) = 0.
  \end{align*}
  Rearrange this equation to write $\Delta_{i+1} = -\beta_{i+1}{p_{i+1}(\lambda_{\min})\over p_i(\lambda_{\min})}$, which can be further rewritten as
  \begin{align*}
    \Delta_{i+1}&\stackrel{\text{\refthm{append:thm:lanczospoly}}}{=} -\beta_{i+1}{(-1)^{i+1}{\det(J_{i+1} - \lambda_{\min} I)/ \prod_{j=1}^{i+1}\beta_j}\over (-1)^i{\det(J_i - \lambda_{\min} I)/ \prod_{j=1}^i\beta_j}}
    = {\det(J_{i+1} - \lambda_{\min} I)\over \det(J_i - \lambda_{\min} I)} = \leftgr{\delta}_{i+1}.\qedhere
  \end{align*}
\end{proof}
\end{lemma}
\begin{remark}
  Lemma~\ref{append:lem:delta} has an implication beyond its utility for the subsequent proofs: it provides a new way of calculating $\alpha_{i+1}$ given the quantities $\leftgr{\delta}_{i+1}$ and $\leftgr{\alpha}_i$; this saves calculation in~\refalgo{append:algo:gql}.
\end{remark}

The following lemma relates $\delta_i$ to $\leftgr{\delta}_i$, which will prove useful in subsequent analysis.
\begin{lemma}\label{append:lem:comp}
Let $\leftgr{\delta}_i$ and $\delta_i$ be computed in the $i$-th iteration of \refalgo{algo:gql}. Then, we have the following:
\begin{align}
  \label{append:eq:7}
  \leftgr{\delta}_i &< \delta_i,\\
  \label{append:eq:8}
  \frac{\leftgr{\delta}_i}{\delta_i} &\le 1 - {\lambda_{\min}\over \lambda_N}.
\end{align}
\end{lemma}
\begin{proof} We prove~\eqref{append:eq:7} by induction. 
Since $\lambda_{\min} > 0$, $\delta_1 = \alpha_1 > \lambda_{\min}$ and $\leftgr{\delta}_1 = \alpha - \lambda_{\min}$ we know that $\leftgr{\delta}_1 < \delta_1$. Assume that $\leftgr{\delta}_i < \delta_i$ is true for all $i\le k$ and considering the $(k+1)$-th iteration:
\begin{equation*}
\leftgr{\delta}_{k+1} = \alpha_{k+1} - \lambda_{\min} - {\beta_k^2\over \leftgr{\delta}_k} < \alpha_{k+1} - {\beta_k^2 \over \delta_k} = \delta_{k+1}.
\end{equation*}

To prove \eqref{append:eq:8}, simply observe the following
\begin{align*}
  \frac{\leftgr{\delta}_i}{\delta_i} &= \frac{\alpha_i - \lambda_{\min} - \beta_{i-1}^2/ \leftgr{\delta}_{i-1}}{\alpha_i - \beta_{i-1}^2/\delta^{i-1}}\quad
\stackrel{\eqref{append:eq:7}}{\le}
\quad \frac{\alpha_i - \lambda_{\min}}{\alpha_i}
\le 1 - \frac{\lambda_{\min}}{\lambda_N}.\qedhere
\end{align*}
\end{proof}

With aforementioned lemmas we will be able to show how fast the difference between $\leftgr{g}_i$ and $g_i$ decays. Note that $\leftgr{g}_i$ gives an upper bound on the objective while $g_i$ gives a lower bound.
\begin{lemma}\label{append:lem:diffconv}
The difference between $\leftgr{g}_i$ and $g_i$ decreases linearly. More specifically we have
\begin{align*}
\leftgr{g}_i - g_i \le 2\kappa^+({\sqrt{\kappa} - 1\over \sqrt{\kappa} + 1})^i g_N
\end{align*}
where $\kappa^+ = \lambda_N/\lambda_{\min}$ and $\kappa$ is the condition number of $A$, i.e., $\kappa = \lambda_N / \lambda_1$.
\begin{proof}
We rewrite the difference $\leftgr{g}_i - g_i$ as follows
\begin{align*}
\leftgr{g}_i &- g_i = {\beta_i^2 c_i^2\over \delta_i(\leftgr{\alpha}_i\delta_i - \beta_i^2)} \\
&= {\beta_i^2 c_i^2\over \delta_i(\alpha_{i+1}\delta_i - \beta_i^2)}{\delta_i(\alpha_{i+1}\delta_i - \beta_i^2)\over \delta_i(\leftgr{\alpha}_i\delta_i - \beta_i^2)}\\
&= {\beta_i^2 c_i^2\over \delta_i(\alpha_{i+1}\delta_i - \beta_i^2)} \frac{1}{\bigl(\leftgr{\alpha}_i - \beta_i^2/ \delta_i\bigr) \big/ \bigl(\alpha_{i+1} - \beta_i^2/ \delta_i\bigr)} \\
&= {\beta_i^2 c_i^2\over \delta_i(\alpha_i\delta_i - \beta_i^2)} {1\over {1 - {\Delta_{i+1} / \delta_{i+1}}}},
\end{align*}
where $\Delta_{i+1} = \alpha_{i+1} - \leftgr{\alpha}_i$. Next, recall that 
$\frac{g_N - g_i}{g_N} \le 2\Bigl(\frac{\sqrt{\kappa} - 1}{\sqrt{\kappa} + 1}\Bigr)^i$. Since $g_i$ lower bounds $g_N$, we have
\begin{align*}
\Bigl(1 - 2\Bigl(\frac{\sqrt{\kappa} - 1}{\sqrt{\kappa} + 1}\Bigr)^i\Bigr) g_N \le g_i \le g_N,\\
\Bigl(1 - 2\Bigl(\frac{\sqrt{\kappa} - 1}{\sqrt{\kappa} + 1}\Bigr)^{i+1}\Bigr) g_N \le g_{i+1} \le g_N.
\end{align*}
Thus, we can conclude that
\begin{align*}
{\beta_i^2 c_i^2\over \delta_i(\alpha_i\delta_i - \beta_i^2)} &= g_{i+1} - g_i
\le 2\Bigl(\frac{\sqrt{\kappa} - 1}{\sqrt{\kappa} + 1}\Bigr)^i g_N.
\end{align*}
Now we focus on the term $\bigl(1 - \Delta_{i+1}/ \delta_{i+1}\bigr)^{-1}$. Using~\reflem{append:lem:delta} we know that $\Delta_{i+1} = \leftgr{\delta}_{i+1}$.
Hence,
\begin{align*}
1 - &\Delta_{i+1}/ \delta_{i+1} = 1 - \leftgr{\delta}_{i+1}/ \delta_{i+1} \\
&\ge 1 - (1 - \lambda_{\min}/\lambda_N) = \lambda_{\min}/ \lambda_N \triangleq\frac{1}{\kappa^+}.
\end{align*}
Finally we have
\begin{equation*}
\leftgr{g}_i - g_i ={\beta_i^2 c_i^2\over \delta_i(\alpha_i\delta_i - \beta_i^2)} \frac{1}{1 - \Delta_{i+1}/ \delta_{i+1}}  
\le 2\kappa^+\Bigl({\sqrt{\kappa} - 1\over \sqrt{\kappa} + 1}\Bigr)^i g_N. \qedhere
\end{equation*}
\end{proof}

\end{lemma}

\begin{theorem}[Relative error of left Gauss-Radau, \refthm{thm:lrconv} in the main text]
\label{append:thm:lrconv}
For left Gauss-Radau quadrature where the preassigned node is $\lambda_{\min}$, we have the following bound on relative error:
\begin{equation*}
  \frac{\leftgr{g}_i - g_N}{g_N} \le 2\kappa^+\Bigl(\frac{\sqrt{\kappa} - 1}{\sqrt{\kappa} + 1}\Bigr)^i,
\end{equation*}
where $\kappa^+ := \lambda_N/ \lambda_{\min},\ i<N$.
\begin{proof}
Write $\leftgr{g}_i = g_i + (\leftgr{g}_i - g_i)$. Since $g_i \le g_N$, using Lemma \ref{append:lem:diffconv} to bound the second term we obtain
\begin{equation*}
  \leftgr{g}_i \le g_N + 2\kappa^+\Bigl({\sqrt{\kappa} - 1\over \sqrt{\kappa} + 1}\Bigr)^i g_N,
\end{equation*}
from which the claim follows upon rearrangement.

\end{proof}

\end{theorem}

Due to the relations between left Gauss-Radau and Gauss-Lobatto, we have the following corollary:

\begin{corollary}[Relative error of Gauss-Lobatto, \refcor{cor:loconv} in the main text]
\label{append:cor:loconv}
For Gauss-Lobatto quadrature, we have the following bound on relative error:
\begin{align}
\frac{\lobatto{g}_i - g_N}{g_N} \le 2\kappa^+\Bigl(\frac{\sqrt{\kappa} - 1}{\sqrt{\kappa} + 1}\Bigr)^{i-1},
\end{align}
where $\kappa^+ := \lambda_N/ \lambda_{\min}\text{ and }i<N$.
\end{corollary}

\section{Generalization: Symmetric Matrices}
\label{append:sec:general}
In this section we consider the case where $u$ lies in the column space of several top eigenvectors of $A$, and discuss how the aforementioned theorems vary. In particular, note that the previous analysis assumes that $A$ is positive definite. With our analysis in this section we relax this assumption to the more general case where $A$ is symmetric with simple eigenvalues, though we require $u$ to lie in the space spanned by eigenvectors of $A$ corresponding to positive eigenvalues.

We consider the case where $A$ is symmetric and has the eigendecomposition of $A = Q \Lambda Q^\top = \sum_{i=1}^{N} \lambda_i q_i q_i^\top$ where $\lambda_i$'s are eigenvalues of $A$ increasing with $i$ and $q_i$'s are corresponding eigenvectors.
Assume that $u$ lies in the column space spanned by top $k$ eigenvectors of $A$ where all these $k$ eigenvectors correspond to positive eigenvalues.
Namely we have $u\in \text{Span}\{\{q_i\}_{i=N-k+1}^N\}$ and $0 < \lambda_{N-k+1}$.

Since we only assume that $A$ is symmetric, it is possible that $A$ is singular and thus we consider the value of $u^\top A^\dagger u$, where $A^\dagger$ is the pseudo-inverse of $A$. Due to the constraints on $u$ we have
\begin{align*}
u^\top A^\dagger u 
= u^\top Q \Lambda^\dagger Q^\top u
= u^\top Q_k \Lambda_k^\dagger Q_k^\top u
= u^\top B^\dagger u,
\end{align*}
where $B = \sum_{i=N-k+1}^N \lambda_i q_i q_i^\top$. Namely, if $u$ lies in the column space spanned by the top $k$ eigenvectors of $A$ then it is equivalent to substitute $A$ with $B$, which is the truncated version of $A$ at top $k$ eigenvalues and corresponding eigenvectors.

Another key observation is that,
given that $u$ lies only in the space spanned by $\{q_i\}_{i=N-k+1}^N$,
the Krylov space starting at $u$ becomes
\begin{align}
\text{Span}\{u, Au, A^2u,\ldots\} = \text{Span}\{u, Bu, B^2u,\ldots, B^{k-1}u\}
\end{align}
This indicates that Lanczos iteration starting at matrix $A$ and vector $u$ will finish constructing the corresponding Krylov space after the $k$-th iteration.
Thus under this condition, \refalgo{algo:gql} will run at most $k$ iterations and then stop. At that time, the eigenvalues of $J_k$ are exactly the eigenvalues of $B$, thus they are exactly $\{\lambda_i\}_{i=N-k+1}^N$ of $A$. Using similar proof as in~\reflem{append:lem:exact}, we can obtain the following generalized exactness result.
\begin{corollary}[Generalized Exactness]
$g_k$, $\rightgr{g}_k$ and $\leftgr{g}_k$ are exact for $u^\top A^\dagger u = u^\top B^\dagger u$, namely
\begin{align*}
g_k = \rightgr{g}_k = \leftgr{g}_k = u^\top A^\dagger u = u^\top B^\dagger u.
\end{align*}
\end{corollary}
The monotonicity and the relations between bounds given by various Gauss-type quadratures will still be the same as in the original case in~Section~\ref{sec:main}, 
but the original convergence rate cannot apply in this case because now we probably have $\lambda_{\min}(B) = 0$, making $\kappa$ undefined.
This crash of convergence rate results from the crash of the convergence of the corresponding conjugate gradient algorithm for solving $Ax = u$.
However, by looking at the proof of, e.g.,~\cite{shewchuk1994introduction}, and by noting that $\lambda_1(B) = \ldots = \lambda_{N-k}(B) = 0$, with a slight modification of the proof we actually obtain the bound
\begin{align*}
\|\eps^i\|_A^2 \le \min_{P_i} \max_{\lambda\in\{\lambda_i\}_{i=N-k+1}^N} [P_i(\lambda)]^2 \|\eps^0\|_A^2,
\end{align*}
where $P_i$ is a polynomial of order $i$. By using properties of Chebyshev polynomials and following the original proof~(e.g.,~\cite{golub2009matrices} or~\cite{shewchuk1994introduction}) we obtain the following lemma for conjugate gradient.
\begin{lemma}
  Let $\eps^k$ be as before (for conjugate gradient). Then,
\begin{align*}
\|\eps^k\|_A \le 2\Bigl(\frac{\sqrt{\kappa'} - 1}{\sqrt{\kappa'} + 1})^k\|\eps_0\|_A,\quad\text{where}\ \kappa' := \lambda_N/\lambda_{N-k+1}.
\end{align*}
\end{lemma}
Following this new convergence rate and connections between conjugate gradient, Lanczos iterations and Gauss quadrature mentioned in Section~\ref{sec:main}, we have the following convergence bounds.
\begin{corollary}[Convergence Rate for Special Case]\label{append:cor:spec_conv}
Under the above assumptions on $A$ and $u$, due to the connection Between Gauss quadrature, Lanczos algorithm and Conjugate Gradient, the relative convergence rates of $g_i$, $\rightgr{g}_i$, $\leftgr{g}_i$ and $\lobatto{g}_i$ are given by
\begin{align*}
{g_k - g_i\over g_k} &\le 2\Bigl({\sqrt{\kappa'} - 1\over \sqrt{\kappa'} + 1}\Bigr)^i\\
{g_k - \rightgr{g}_i\over g_k} &\le 2\Bigl({\sqrt{\kappa'} - 1\over \sqrt{\kappa'} + 1}\Bigr)^i\\
{\leftgr{g}_i - g_k\over g_k} &\le 2\kappa_m'\Bigl({\sqrt{\kappa'} - 1\over \sqrt{\kappa'} + 1}\Bigr)^i\\
{\lobatto{g}_i - g_k\over g_k} &\le 2\kappa_m'\Bigl({\sqrt{\kappa'} - 1\over \sqrt{\kappa'} + 1}\Bigr)^i,
\end{align*}
where $\kappa_m' = \lambda_N/\lambda_{\min}'$ and $0 < \lambda_{\min}' < \lambda_{N-k+1}$ is a lowerbound for nonzero eigenvalues of $B$.
\end{corollary}

\section{Accelerating MCMC for $k$-\dpp}\label{append:sec:kdppalgo}
We present details of a \emph{Retrospective Markov Chain Monte Carlo (MCMC)} in~\refalgo{append:algo:gausskdpp} and~\refalgo{append:algo:kdpp_judge} that samples for efficiently drawing samples from a $k$-\dpp, by accelerating it using our results on Gauss-type quadratures.

\begin{algorithm}
	\caption{Gauss-$k$\dpp($L,k$)}\label{append:algo:gausskdpp}
	\begin{algorithmic} 
	\REQUIRE{$L$ the kernel matrix we want to sample \dpp from, $k$ the size of subset and $\cy = [N]$ the ground set}
	\ENSURE{$Y$ sampled from exact $k$\dpp($L$) where $|Y| = k$}
	\STATE Randomly Initialize $Y\subseteq \cy$ where $|Y| = k$
	\WHILE{not mixed}
		\STATE Pick $v\in Y$ and $u\in\cy\backslash Y$ uniformly randomly
		\STATE Pick $p\in(0,1)$ uniformly randomly	
		\STATE $Y' = Y\backslash\{v\}$	
		\STATE Get lower and upper bounds $\lambda_{\min}$, $\lambda_{\max}$ of the spectrum of $L_{Y'}$
		\IF{$k$-\dpp-JudgeGauss($pL_{v,v} - L_{u,u}, p, L_{Y',u}, L_{Y',v}, \lambda_{\min}, \lambda_{\max}$) = {\it True}}
			\STATE $Y = Y'\cup\{u\}$
		\ENDIF
	\ENDWHILE
\end{algorithmic}
\end{algorithm}

\begin{algorithm}
	\caption{$k$\dpp-JudgeGauss($t, p, u, v, A, \lambda_{\min}, \lambda_{\max}$)}\label{append:algo:kdpp_judge}
	\begin{algorithmic} 
	\REQUIRE{$t$ the target value, $p$ the scaling factor, $u$, $v$ and $A$ the corresponding vectors and matrix, $\lambda_{\min}$ and $\lambda_{\max}$ lower and upper bounds for the spectrum of $A$}
	\ENSURE{Return \textit{True} if $t < p (v^\top A^{-1} v) -  u^\top A^{-1} u$, \textit{False} if otherwise}
	\STATE $u_{-1} = 0$, $u_0 = u / \|u\|$, $i^u = 1$, $\beta_{0}^u = 0$, $d^u = \infty$
	\STATE $v_{-1} = 0$, $v_0 = v / \|v\|$, $i^v = 1$, $\beta_0^v = 0$, $d^v = \infty$
	\WHILE{{\it True}}
		\IF{$d^u > pd^v$}
    		\STATE Run one more iteration of Gauss-Radau on $u^\top A^{-1}u$ to get tighter $(\leftgr{g})^u$ and $(\rightgr{g})^u$
		    \STATE $d^u = (\leftgr{g})^u - (\rightgr{g})^u$
		\ELSE
			\STATE Run one more iteration of Gauss-Radau on $v^\top A^{-1}v$ to get tighter $(\leftgr{g})^v$ and $(\rightgr{g})^v$
    		\STATE $d^v = (\leftgr{g})^v - (\rightgr{g})^v$
		\ENDIF
		\IF{$t < p\|v\|^2(\rightgr{g})^v - \|u\|^2 (\leftgr{g})^u$}
			\STATE Return \textit{True}
		\ELSIF{$t \ge p\|v\|^2(\leftgr{g})^v - \|u\|^2 (\rightgr{g})^u$}
			\STATE Return \textit{False}
		\ENDIF
	\ENDWHILE
\end{algorithmic}
\end{algorithm}

\section{Accelerating Stochastic Double Greedy}\label{append:sec:dgalgo}

We present details of \emph{Retrospective Stochastic Double Greedy} in~\refalgo{append:algo:gaussdg} and~\refalgo{append:algo:dg_judge} that efficiently select a subset $Y\in\cy$ that approximately maximize $\log\det(L_Y)$.

\begin{algorithm}
	\caption{Gauss-\dg($L$)}\label{append:algo:gaussdg}
	\begin{algorithmic} 
	\REQUIRE{$L$ the kernel matrix and $\cy = [N]$ the ground set}
	\ENSURE{$X\in\cy$ that approximately maximize $\log\det(L_Y)$}
	\STATE $X_0 = \emptyset$, $Y_0=\cy$
	\FOR{$i=1,2,\ldots,N$}
		\STATE $Y_i' = Y_{i-1}\backslash\{i\}$
		\STATE Sample $p\in(0,1)$ uniformly randomly
		\STATE Get lower and upper bounds $\lambda_{\min}^-,\lambda_{\max}^-,\lambda_{\min}^+,\lambda_{\max}^+$ of the spectrum of $L_{X_{i-1}}$ and $L_{Y_i'}$ respectively
		\IF{\dg-JudgeGauss($L_{X_{i-1}}, L_{Y_i'}, L_{X_{i-1},i}, L_{Y_i',i}, L_{i,i}, p, \lambda_{\min}^-,\lambda_{\max}^-,\lambda_{\min}^+,\lambda_{\max}^+$) = {\it True}}
			\STATE $X_i = X_{i-1}\cup\{i\}$
		\ELSE
			\STATE $Y_i = Y_i'$
		\ENDIF
	\ENDFOR
\end{algorithmic}
\end{algorithm}

\begin{algorithm}
	\caption{\dg-JudgeGauss($A, B, u, v, t, p, \lambda_{\min}^A,\lambda_{\max}^A,\lambda_{\min}^B,\lambda_{\max}^B$)}\label{append:algo:dg_judge}
	\begin{algorithmic} 
	\REQUIRE{$t$ the target value, $p$ the scaling factor, $u$, $v$, $A$ and $B$ the corresponding vectors and matrix, $\lambda_{\min}^A$, $\lambda_{\max}^A$, $\lambda_{\min}^B$, $\lambda_{\max}^B$ lower and upper bounds for the spectrum of $A$ and $B$}
	\ENSURE{Return \textit{True} if $p |\log(t - u^\top A^{-1} u)|_+ \le (1-p) |-\log(t-v^\top B^{-1} v)|_+$, \textit{False} if otherwise}
	\STATE $d^u = \infty$, $d^v = \infty$
	\WHILE{{\it True}}
		\IF{$pd^u > (1-p)d^v$}
    		\STATE Run one more iteration of Gauss-Radau on $u^\top A^{-1}u$ to get tighter lower and upper bounds $l^u$, $u^u$ for $|\log(t - u^\top A^{-1} u)|_+$
		    \STATE $d^u = u^u - l^u$
		\ELSE
			\STATE Run one more iteration of Gauss-Radau on $v^\top B^{-1}v$ to get tighter lower and upper bounds $l^v$, $u^v$ for $|\log(t - v^\top B^{-1} v)|_+$
    		\STATE $d^v = u^v - l^v$
		\ENDIF
		\IF{$pu^u\le (1-p)l^v$}
			\STATE Return \textit{True}
		\ELSIF{$pl^u > (1-p)u^v$}
			\STATE Return \textit{False}
		\ENDIF
	\ENDWHILE
\end{algorithmic}
\end{algorithm}

\end{appendix}


\begin{thebibliography}{62}
\providecommand{\natexlab}[1]{#1}
\providecommand{\url}[1]{\texttt{#1}}
\expandafter\ifx\csname urlstyle\endcsname\relax
  \providecommand{\doi}[1]{doi: #1}\else
  \providecommand{\doi}{doi: \begingroup \urlstyle{rm}\Url}\fi

\bibitem[Anari et~al.(2016)Anari, Gharan, and Rezaei]{anari2016monte}
Anari, Nima, Gharan, Shayan~Oveis, and Rezaei, Alireza.
\newblock {M}onte {C}arlo {M}arkov chain algorithms for sampling strongly
  {R}ayleigh distributions and determinantal point processes.
\newblock In \emph{COLT}, 2016.

\bibitem[Atzori et~al.(2010)Atzori, Iera, and Morabito]{atzori2010internet}
Atzori, Luigi, Iera, Antonio, and Morabito, Giacomo.
\newblock The {I}nternet of {T}hings: A survey.
\newblock \emph{Computer networks}, 54\penalty0 (15):\penalty0 2787--2805,
  2010.

\bibitem[Bai \& Golub(1996)Bai and Golub]{bai1996bounds}
Bai, Zhaojun and Golub, Gene~H.
\newblock Bounds for the trace of the inverse and the determinant of symmetric
  positive definite matrices.
\newblock \emph{Annals of Numerical Mathematics}, pp.\  29--38, 1996.

\bibitem[Bai et~al.(1996)Bai, Fahey, and Golub]{bai1996some}
Bai, Zhaojun, Fahey, Gark, and Golub, Gene~H.
\newblock Some large-scale matrix computation problems.
\newblock \emph{Journal of Computational and Applied Mathematics}, pp.\
  71--89, 1996.

\bibitem[Bekas et~al.(2007)Bekas, Kokiopoulou, and Saad]{bekas2007estimator}
Bekas, Constantine, Kokiopoulou, Effrosyni, and Saad, Yousef.
\newblock An estimator for the diagonal of a matrix.
\newblock \emph{Applied numerical mathematics}, pp.\  1214--1229, 2007.

\bibitem[Bekas et~al.(2009)Bekas, Curioni, and Fedulova]{bekas2009low}
Bekas, Constantine, Curioni, Alessandro, and Fedulova, Irina.
\newblock Low cost high performance uncertainty quantification.
\newblock In \emph{Proceedings of the 2nd Workshop on High Performance
  Computational Finance}, 2009.

\bibitem[Belabbas \& Wolfe(2009)Belabbas and Wolfe]{belabbas2009spectral}
Belabbas, Mohamed-Ali and Wolfe, Patrick~J.
\newblock Spectral methods in machine learning and new strategies for very
  large datasets.
\newblock \emph{Proceedings of the National Academy of Sciences}, pp.\
  369--374, 2009.

\bibitem[Benzi \& Golub(1999)Benzi and Golub]{benzi1999bounds}
Benzi, Michele and Golub, Gene~H.
\newblock Bounds for the entries of matrix functions with applications to
  preconditioning.
\newblock \emph{BIT Numerical Mathematics}, pp.\  417--438, 1999.

\bibitem[Benzi \& Klymko(2013)Benzi and Klymko]{benzi2013total}
Benzi, Michele and Klymko, Christine.
\newblock Total communicability as a centrality measure.
\newblock \emph{J. Complex Networks}, pp.\  124--149, 2013.

\bibitem[Bonacich(1987)]{bonacich87}
Bonacich, Phillip.
\newblock Power and centrality: A family of measures.
\newblock \emph{American Journal of Sociology}, pp.\  1170--1182, 1987.

\bibitem[Boutsidis et~al.(2009)Boutsidis, Mahoney, and
  Drineas]{boutsidis2009improved}
Boutsidis, Christos, Mahoney, Michael~W., and Drineas, Petros.
\newblock An improved approximation algorithm for the column subset selection
  problem.
\newblock In \emph{SODA}, pp.\  968--977, 2009.

\bibitem[Brezinski(1999)]{brezinski1999error}
Brezinski, Claude.
\newblock Error estimates for the solution of linear systems.
\newblock \emph{SIAM Journal on Scientific Computing}, pp.\  764--781, 1999.

\bibitem[Brezinski et~al.(2012)Brezinski, Fika, and
  Mitrouli]{brezinski2012estimations}
Brezinski, Claude, Fika, Paraskevi, and Mitrouli, Marilena.
\newblock Estimations of the trace of powers of positive self-adjoint operators
  by extrapolation of the moments.
\newblock \emph{Electronic Transactions on Numerical Analysis}, pp.\  144--155,
  2012.

\bibitem[Buchbinder et~al.(2012)Buchbinder, Feldman, Naor, and
  Schwartz]{buchbinder12}
Buchbinder, Niv, Feldman, Moran, Naor, Joseph, and Schwartz, Roy.
\newblock A tight linear time (1/2)-approximation for unconstrained submodular
  maximization.
\newblock In \emph{FOCS}, 2012.

\bibitem[Dong \& Liu(1994)Dong and Liu]{dong1994stochastic}
Dong, Shao-Jing and Liu, Keh-Fei.
\newblock Stochastic estimation with $z_2$ noise.
\newblock \emph{Physics Letters B}, pp.\  130--136, 1994.

\bibitem[Estrada \& Higham(2010)Estrada and Higham]{estrada2010network}
Estrada, Ernesto and Higham, Desmond~J.
\newblock Network properties revealed through matrix functions.
\newblock \emph{SIAM Review}, pp.\  696--714, 2010.

\bibitem[Fenu et~al.(2013)Fenu, Martin, Reichel, and
  Rodriguez]{fenu2013network}
Fenu, Caterina, Martin, David~R., Reichel, Lothar, and Rodriguez, Giuseppe.
\newblock Network analysis via partial spectral factorization and {G}auss
  quadrature.
\newblock \emph{SIAM Journal on Scientific Computing}, pp.\  A2046--A2068,
  2013.

\bibitem[Fika \& Koukouvinos(2015)Fika and Koukouvinos]{fika2015stochastic}
Fika, Paraskevi and Koukouvinos, Christos.
\newblock Stochastic estimates for the trace of functions of matrices via
  {H}adamard matrices.
\newblock \emph{Communications in Statistics-Simulation and Computation}, 2015.

\bibitem[Fika \& Mitrouli(2015)Fika and Mitrouli]{fika2015estimation}
Fika, Paraskevi and Mitrouli, Marilena.
\newblock Estimation of the bilinear form {$y^*f(A)x$} for {H}ermitian
  matrices.
\newblock \emph{Linear Algebra and its Applications}, 2015.

\bibitem[Fika et~al.(2014)Fika, Mitrouli, and Roupa]{fika2014estimates}
Fika, Paraskevi, Mitrouli, Marilena, and Roupa, Paraskevi.
\newblock Estimates for the bilinear form {$x^TA^{-1}y$} with applications to
  linear algebra problems.
\newblock \emph{Electronic Transactions on Numerical Analysis}, pp.\  70--89,
  2014.

\bibitem[Freericks(2006)]{freericks2006transport}
Freericks, James~K.
\newblock Transport in multilayered nanostructures.
\newblock \emph{The Dynamical Mean-Field Theory Approach, Imperial College,
  London}, 2006.

\bibitem[Frommer et~al.(2012)Frommer, Lippert, Medeke, and
  Schilling]{frommer2000numerical}
Frommer, Andreas, Lippert, Thomas, Medeke, Bj{\"o}rn, and Schilling, Klaus.
\newblock \emph{Numerical Challenges in Lattice Quantum Chromodynamics: Joint
  Interdisciplinary Workshop of John Von Neumann Institute for Computing,
  J{\"u}lich, and Institute of Applied Computer Science, Wuppertal University,
  August 1999}, volume~15.
\newblock Springer Science \& Business Media, 2012.

\bibitem[Gauss(1815)]{gauss1815methodus}
Gauss, Carl~F.
\newblock \emph{Methodus nova integralium valores per approximationem
  inveniendi}.
\newblock apvd Henricvm Dieterich, 1815.

\bibitem[Gautschi(1981)]{gautschi1981survey}
Gautschi, Walter.
\newblock A survey of {G}auss-{C}hristoffel quadrature formulae.
\newblock In \emph{EB Christoffel}, pp.\  72--147. Springer, 1981.

\bibitem[Gillenwater et~al.(2012)Gillenwater, Kulesza, and
  Taskar]{gillenwater12}
Gillenwater, Jennifer, Kulesza, Alex, and Taskar, Ben.
\newblock Near-optimal {MAP} inference for determinantal point processes.
\newblock In \emph{NIPS}, 2012.

\bibitem[Gittens \& Mahoney(2013)Gittens and Mahoney]{gittens2013revisiting}
Gittens, Alex and Mahoney, Michael~W.
\newblock Revisiting the {N}ystr{\"o}m method for improved large-scale machine
  learning.
\newblock \emph{ICML}, 2013.

\bibitem[Golub(1973)]{golub1973some}
Golub, Gene~H.
\newblock Some modified matrix eigenvalue problems.
\newblock \emph{SIAM Review}, pp.\  318--334, 1973.

\bibitem[Golub \& Meurant(1997)Golub and Meurant]{golub1997matrices}
Golub, Gene~H. and Meurant, G{\'e}rard.
\newblock Matrices, moments and quadrature {II}; how to compute the norm of the
  error in iterative methods.
\newblock \emph{BIT Numerical Mathematics}, pp.\  687--705, 1997.

\bibitem[Golub \& Meurant(2009)Golub and Meurant]{golub2009matrices}
Golub, Gene~H. and Meurant, G{\'e}rard.
\newblock \emph{Matrices, moments and quadrature with applications}.
\newblock Princeton University Press, 2009.

\bibitem[Golub \& Welsch(1969)Golub and Welsch]{golub1969calculation}
Golub, Gene~H. and Welsch, John~H.
\newblock Calculation of {G}auss quadrature rules.
\newblock \emph{Mathematics of computation}, pp.\  221--230, 1969.

\bibitem[Golub et~al.(2008)Golub, Stoll, and Wathen]{golub2008approximation}
Golub, Gene~H., Stoll, Martin, and Wathen, Andy.
\newblock Approximation of the scattering amplitude and linear systems.
\newblock \emph{Elec. Tran. on Numerical Analysis}, pp.\  178--203, 2008.

\bibitem[Hestenes \& Stiefel(1952)Hestenes and Stiefel]{hestenes1952methods}
Hestenes, Magnus~R. and Stiefel, Eduard.
\newblock Methods of conjugate gradients for solving linear systems.
\newblock \emph{J. Research of the National Bureau of Standards}, pp.\
  409--436, 1952.

\bibitem[Hough et~al.(2006)Hough, Krishnapur, Peres, and Vir{\'a}g]{hough2005}
Hough, J.~Ben, Krishnapur, Manjunath, Peres, Yuval, and Vir{\'a}g, B{\'a}lint.
\newblock Determinantal processes and independence.
\newblock \emph{Probability Surveys}, 2006.

\bibitem[Kang(2013)]{kang2013fast}
Kang, Byungkon.
\newblock Fast determinantal point process sampling with application to
  clustering.
\newblock In \emph{NIPS}, pp.\  2319--2327, 2013.

\bibitem[Krause et~al.(2008)Krause, Singh, and Guestrin]{krause2008}
Krause, Andreas, Singh, Ajit, and Guestrin, Carlos.
\newblock Near-optimal sensor placements in {G}aussian {p}rocesses: {T}heory,
  efficient algorithms and empirical studies.
\newblock \emph{JMLR}, pp.\  235--284, 2008.

\bibitem[Kulesza \& Taskar(2012)Kulesza and Taskar]{kulesza2012determinantal}
Kulesza, Alex and Taskar, Ben.
\newblock Determinantal point processes for machine learning.
\newblock \emph{arXiv:1207.6083}, 2012.

\bibitem[Kwok \& Adams(2012)Kwok and Adams]{kwok2012priors}
Kwok, James~T. and Adams, Ryan~P.
\newblock Priors for diversity in generative latent variable models.
\newblock In \emph{NIPS}, pp.\  2996--3004, 2012.

\bibitem[Lanczos(1950)]{lanczos1950iteration}
Lanczos, Cornelius.
\newblock \emph{An iteration method for the solution of the eigenvalue problem
  of linear differential and integral operators}.
\newblock United States Governm. Press Office Los Angeles, CA, 1950.

\bibitem[Lee et~al.(2014)Lee, Ozdaglar, and Shah]{christina2014solving}
Lee, Christina~E., Ozdaglar, Asuman~E., and Shah, Devavrat.
\newblock Solving systems of linear equations: Locally and asynchronously.
\newblock \emph{arXiv}, abs/1411.2647, 2014.

\bibitem[Leskovec et~al.(2008)Leskovec, Lang, Dasgupta, and Mahoney]{jure08}
Leskovec, Jure, Lang, Kevin~J., Dasgupta, Anirban, and Mahoney, Michael~W.
\newblock Statistical properties of community structure in large social and
  information networks.
\newblock In \emph{WWW}, pp.\  695--704, 2008.

\bibitem[Lin et~al.(2011{\natexlab{a}})Lin, Yang, Lu, and Ying]{lin2011fast}
Lin, Lin, Yang, Chao, Lu, Jianfeng, and Ying, Lexing.
\newblock A fast parallel algorithm for selected inversion of structured sparse
  matrices with application to 2{D} electronic structure calculations.
\newblock \emph{SIAM Journal on Scientific Computing}, pp.\  1329--1351,
  2011{\natexlab{a}}.

\bibitem[Lin et~al.(2011{\natexlab{b}})Lin, Yang, Meza, Lu, Ying, and
  E]{lin2011selinv}
Lin, Lin, Yang, Chao, Meza, Juan~C., Lu, Jianfeng, Ying, Lexing, and E, Weinan.
\newblock Selinv--an algorithm for selected inversion of a sparse symmetric
  matrix.
\newblock \emph{ACM Transactions on Mathematical Software}, 2011{\natexlab{b}}.

\bibitem[Lobatto(1852)]{lobatto1852lessen}
Lobatto, Rehuel.
\newblock \emph{Lessen over de differentiaal-en integraal-rekening: Dl. 2
  Integraal-rekening}, volume~1.
\newblock Van Cleef, 1852.

\bibitem[Meurant(1997)]{meurant1997computation}
Meurant, G{\'e}rard.
\newblock The computation of bounds for the norm of the error in the conjugate
  gradient algorithm.
\newblock \emph{Numerical Algorithms}, pp.\  77--87, 1997.

\bibitem[Meurant(1999)]{meurant1999numerical}
Meurant, G{\'e}rard.
\newblock Numerical experiments in computing bounds for the norm of the error
  in the preconditioned conjugate gradient algorithm.
\newblock \emph{Numerical Algorithms}, pp.\  353--365, 1999.

\bibitem[Meurant(2006)]{rard2006lanczos}
Meurant, G{\'e}rard.
\newblock \emph{The {L}anczos and conjugate gradient algorithms: from theory to
  finite precision computations}, volume~19.
\newblock SIAM, 2006.

\bibitem[Minoux(1978)]{minoux78}
Minoux, Michel.
\newblock Accelerated greedy algorithms for maximizing submodular set
  functions.
\newblock In \emph{Optimization Techniques}, pp.\  234--243. Springer, 1978.

\bibitem[Mirzasoleiman et~al.(2015)Mirzasoleiman, Badanidiyuru, Karbasi,
  Vondr{\'{a}}k, and Krause]{mirzasoleiman15}
Mirzasoleiman, Baharan, Badanidiyuru, Ashwinkumar, Karbasi, Amin,
  Vondr{\'{a}}k, Jan, and Krause, Andreas.
\newblock Lazier than lazy greedy.
\newblock In \emph{AAAI}, 2015.

\bibitem[Nemhauser et~al.(1978)Nemhauser, Wolsey, and Fisher]{nemhauser1978}
Nemhauser, George~L.., Wolsey, Laurence~A., and Fisher, Marshall~L.
\newblock An analysis of approximations for maximizing submodular set
  functions--{I}.
\newblock \emph{Mathematical Programming}, pp.\  265--294, 1978.

\bibitem[Page et~al.(1999)Page, Brin, Motwani, and Winograd]{page99rank}
Page, Lawrence, Brin, Sergey, Motwani, Rajeev, and Winograd, Terry.
\newblock \emph{The PageRank citation ranking: bringing order to the web.}
\newblock Stanford InfoLab, 1999.

\bibitem[Radau(1880)]{radau1880etude}
Radau, Rodolphe.
\newblock {\'E}tude sur les formules d'approximation qui servent {\`a} calculer
  la valeur num{\'e}rique d'une int{\'e}grale d{\'e}finie.
\newblock \emph{J. de Math{\'e}matiques Pures et Appliqu{\'e}es}, pp.\
  283--336, 1880.

\bibitem[Rasmussen \& Williams(2006)Rasmussen and Williams]{rasmussen}
Rasmussen, Carl~E. and Williams, Christopher K.~I.
\newblock \emph{Gaussian {P}rocesses for Machine Learning}.
\newblock MIT Press, Cambridge, MA, 2006.

\bibitem[Rockov{\'a} \& George(2015)Rockov{\'a} and
  George]{rockovadeterminantal}
Rockov{\'a}, Veronika and George, Edward~I.
\newblock Determinantal priors for variable selection, 2015.

\bibitem[Scott(2012)]{scott2012social}
Scott, John.
\newblock \emph{Social network analysis}.
\newblock Sage, 2012.

\bibitem[Sherman \& Morrison(1950)Sherman and Morrison]{sherman1950adjustment}
Sherman, Jack and Morrison, Winifred~J.
\newblock Adjustment of an inverse matrix corresponding to a change in one
  element of a given matrix.
\newblock \emph{The Annals of Mathematical Statistics}, pp.\  124--127, 1950.

\bibitem[Shewchuk(1994)]{shewchuk1994introduction}
Shewchuk, Jonathan~R.
\newblock An introduction to the conjugate gradient method without the
  agonizing pain, 1994.

\bibitem[Sidje \& Saad(2011)Sidje and Saad]{sidje2011rational}
Sidje, Roger~B. and Saad, Yousef.
\newblock Rational approximation to the fermi--dirac function with applications
  in density functional theory.
\newblock \emph{Numerical Algorithms}, pp.\  455--479, 2011.

\bibitem[Stoer \& Bulirsch(2013)Stoer and Bulirsch]{stoer2013introduction}
Stoer, Josef and Bulirsch, Roland.
\newblock \emph{Introduction to numerical analysis}, volume~12.
\newblock Springer Science \& Business Media, 2013.

\bibitem[Sviridenko et~al.(2015)Sviridenko, Vondr\'ak, and Ward]{sviridenko15}
Sviridenko, Maxim, Vondr\'ak, Jan, and Ward, Justin.
\newblock Optimal approximation for submodular and supermodular optimization
  with bounded curvature.
\newblock In \emph{SODA}, 2015.

\bibitem[Tang \& Saad(2012)Tang and Saad]{tang2012probing}
Tang, Jok~M. and Saad, Yousef.
\newblock A probing method for computing the diagonal of a matrix inverse.
\newblock \emph{Numerical Linear Algebra with Applications}, pp.\  485--501,
  2012.

\bibitem[Wasow(1952)]{wasow1952note}
Wasow, Wolfgang~R.
\newblock A note on the inversion of matrices by random walks.
\newblock \emph{Mathematical Tables and Other Aids to Computation}, pp.\
  78--81, 1952.

\bibitem[Wilf(1962)]{wilf1962mathematics}
Wilf, Herbert~S.
\newblock \emph{Mathematics for the Physical Sciences}.
\newblock Wiley, New York, 1962.

\end{thebibliography}
\end{document}